\documentclass{article}

\usepackage{microtype}
\usepackage{graphicx}
\usepackage{subfigure}
\usepackage{booktabs} 

\usepackage{hyperref}

\usepackage[accepted]{icml2024}

\usepackage{amsmath}
\usepackage{amssymb}
\usepackage{mathtools}
\usepackage{amsthm}

\usepackage[capitalize,noabbrev]{cleveref}

\usepackage{shortcuts}

\theoremstyle{plain}
\newtheorem{theorem}{Theorem}[section]

\theoremstyle{definition}

\newtheorem{assumption}[theorem]{Assumption}
\theoremstyle{remark}

\usepackage[textsize=tiny]{todonotes}

\allowdisplaybreaks
\icmltitlerunning{Prediction-powered Generalization of Causal Inferences}

\begin{document}

\twocolumn[
\icmltitle{Prediction-powered Generalization of Causal Inferences}



\icmlsetsymbol{equal}{*}

\begin{icmlauthorlist}
\icmlauthor{Ilker Demirel}{mit,broad}
\icmlauthor{Ahmed Alaa}{ucb}
\icmlauthor{Anthony Philippakis}{broad}
\icmlauthor{David Sontag}{mit}
\end{icmlauthorlist}

\icmlaffiliation{ucb}{Department of Computational Precision Health, UC Berkeley and UCSF}
\icmlaffiliation{broad}{Eric and Wendy Schmidt Center, Broad Institute of MIT and Harvard}
\icmlaffiliation{mit}{MIT CSAIL}

\icmlcorrespondingauthor{Ilker Demirel}{demirel@mit.edu}

\icmlkeywords{Machine Learning, Causal Inference, Clinical Trials, Observational Studies, Experimental Studies, Combining Evidence, ICML}

\vskip 0.3in
]



\printAffiliationsAndNotice{} 

\begin{abstract}
    Causal inferences from a randomized controlled trial (RCT) may not pertain to a {\em target} population where some effect modifiers have a different distribution. Prior work studies {\em generalizing} the results of a trial to a target population with no outcome but covariate data available. We show how the limited size of trials makes generalization a statistically infeasible task, as it requires estimating complex nuisance functions. We develop generalization algorithms that supplement the trial data with a prediction model learned from an additional {\em observational} study (OS), without making {\em any} assumptions on the OS. We theoretically and empirically show that our methods facilitate better generalization when the OS is \enquote{high-quality}, and remain robust when it is not, and {\em e.g.}, have unmeasured confounding.
\end{abstract}

\section{Introduction}
\label{sec:intro}
Experimental data from randomized controlled trials (RCT) is the gold standard for causal inference as various biases are avoided by design \cite{imbens2015causal, Hernan2021-yd}. However, in addition to being time and cost-intensive, RCTs often exhibit limited external validity, and their findings may not apply to a {\em target population} \cite{rothwell2005external, stuart2011use}. The generalizability of an RCT is compromised when baseline factors that influence prognosis ({\em effect modifiers}) have different distributions in the trial and target populations \cite{dahabreh2019generalizing} (see \Cref{fig:1}). For instance, trials may consist of healthier individuals on average than routine clinical practice. Since the overall health status likely affects the prognosis, it leads to \enquote{confounding} bias between the population-level effects in the trial and the target populations \cite{hernan2004structural}.

\citet{dahabreh2019generalizing, dahabreh2020extending} develop methods that use individual-level covariate, treatment, and outcome data from a trial and only covariate information from the target population to estimate causal quantities in the latter (generalization). In this work, we show how combining trial data with potentially biased observational data, {\em e.g.}, found from electronic health records, can power better generalization.  

\paragraph{Our Contributions}
We derive the generalization mean-squared error when an outcome model learned from the trial is used in the target sample to estimate an average causal effect in the target population, and probe how it increases when the trial is small and not representative of the target population (\Cref{sec:LFTG}). Drawing inspiration from recent advances in using black-box models for valid statistical inference \cite{schuler2021increasing, angelopoulos2023prediction}, we develop {\em prediction-powered} estimators that leverage additional observational data without {\em any} assumptions on it and discuss when they lead to lower generalization error (\Cref{sec:lppod}). We simulate over a thousand data-generating processes and find that our estimators yield remarkable improvements when the observational data is high-quality and maintain baseline performance when it is not (\Cref{sec:synth}).

\paragraph{Related Work} 
There is growing interest in integrating data from trials and observational studies (OS) \cite{bareinboim2016causal, yang2019combining, nice2022uk, colnet2024causal}. \citet{schuler2021increasing, liao2023transfer} show how adjustment by the predictions of a model learned from an OS can increase power in analyzing a trial. Similarly, \citet{guo2021multi} investigate how coupling trial data and with \enquote{control-variates} constructed in an OS may enable smaller-variance estimation of the average treatment effect (ATE) in the {\em trial population}. \citet{Hartman2015-td, degtiar2023conditional} study generalization to a target population defined by the OS population or its union with the trial population. \citet{han2023multiply} study ATE estimation in a target population by incorporating data from multiple {\em source} populations, where the ATE is identifiable in {\em all} of the populations but different. \citet{oberst2023understanding} review methods that combine ATE estimates from a trial and an OS to obtain a better hybrid estimate \cite{rosenman2020combining, Cheng2021-sn, yang2023elastic}. \citet{kallus2018removing, Chen2021-eo, hatt2022combining} consider the heterogeneity in effects and focus on the conditional ATE (CATE) {\em function}. \citet{rosenman2021designing} adopt a different angle and studies more efficient trial design using data from OS.

Another line of papers studies {\em benchmarking} evidence from OS \cite{forbes2020benchmarking}. \citet{hussain2022falsification, hussain2023falsification, demirel2024benchmarking} develop falsification tests for the causal assumptions by comparing the findings of an OS and a trial. \citet{de2024detecting, de2024hidden} focus on {\em quantifying} the hidden confounding in an OS, and \citet{karlsson2024detecting} show how one can {\em detect} hidden confounding using multiple OS with a shared data-generating process.

In the works above, the target population of interest is taken as either the OS population or its union with the trial population. We consider a more general setting where the target population is defined separately from the trial and OS populations so long as it consists of trial-eligible individuals. For instance, the target population can represent a subgroup in the trial with a small sample size.

Our goal is to estimate population-level causal effects in the target population, for which only covariate information is available, by integrating data from a small trial and a large OS. We detail the necessary causal identification assumptions in \Cref{sec:background}, which crucially do not enforce any unverifiable conditions on the OS, and describe our estimators in \Cref{sec:lppod}. We do not go the route of cooking a recipe to combine real-valued estimates from the trial and the OS, nor promise to give guarantees on the granular CATE function, as the former offers poor flexibility in utilizing rich observational data and the latter replaces the {\em causal} assumptions on the OS with statistical assumptions on its bias function. Our approach lies somewhere in between: we fit an outcome {\em function} from the OS using flexible machine learning models, which can be subject to {\em causal biases}, and analyze how coupling it with trial data can power the estimation of a {\em real-valued} causal estimand in the target population.
\section{Background}
\label{sec:background}
\subsection{Notation and Objective}
\label{sec:back}
We consider a {\em nested} design where a trial is sampled from an underlying population of trial-eligible individuals. Note that our methods can easily be extended to {\em nonnested} designs where the target sample is obtained separately; {\em e.g.}, to represent a subgroup for which the trial sample alone cannot power statistically significant inference.

We have access to an i.i.d. sample of observations ${\cal D} = \{ W_i \}_{i=1}^n$ with $W_i = (X_i, S_i, S_i \times A_i, S_i \times Y_i)$, where $X \in {\cal X}$ is a $d$-dimensional covariate vector, $S$ is a binary trial participation indicator, $A$ is a categorical treatment, and $Y \in \mathbb{R}$ is the outcome of interest. Only covariate data is available for non-participants ($S=0$), while treatment and outcome data are also available for participants ($S=1$). 

The {\em target} population of interest is represented by non-participants. We denote by ${\cal D}_1 \subset {\cal D}$ the set of trial participants and by ${\cal D}_0 \subset {\cal D}$ the set of non-participants with sizes $n_1 = \sum_{i=1}^n \ind{S_i=1}$ and $n_0 = \sum_{i=1}^n \ind{S_i=0}$, partitioning the composite sample ${\cal D}$. Further, we denote by $P_1$ and $P_0$ the joint distribution of $W$ in the underlying trial and target populations. For instance, $X \sim P_0$ represents a covariate drawn from the target distribution $P(X \mid S=0)$.
\begin{figure}[t]
    \centering
    \includegraphics[width=\linewidth]{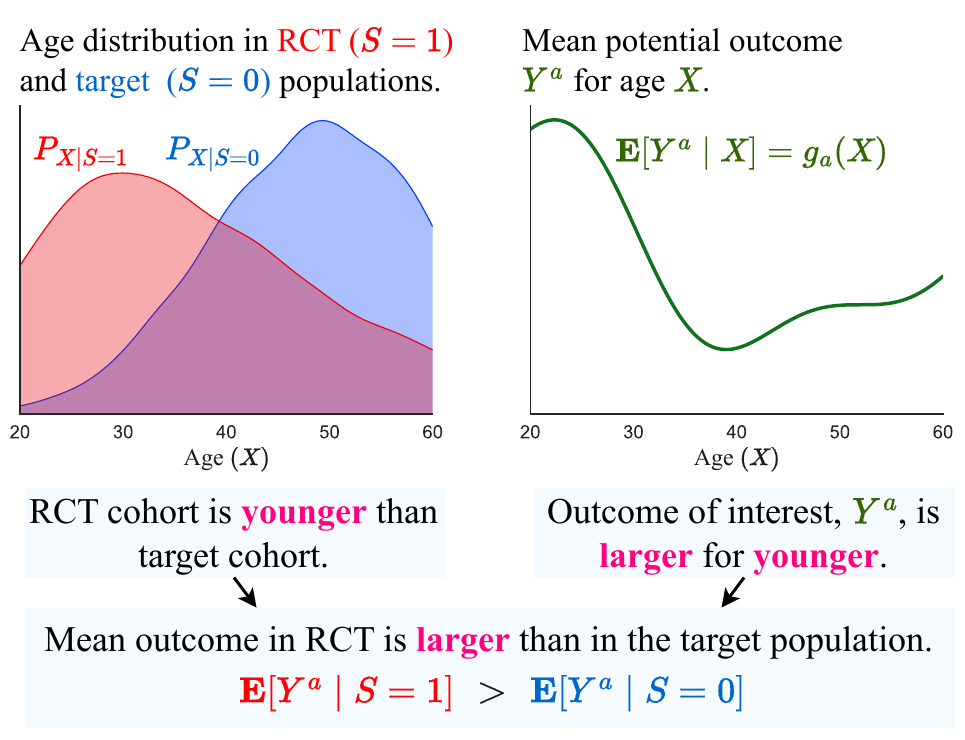}
    \vspace{-10pt}
    \caption{Age influences both selection into the trial and the outcome, inducing confounding bias between the population-level mean potential outcomes in the trial and target populations.}
    \label{fig:1}
\end{figure}

We seek causal inference in the target population. Specifically, denoting by $Y^a$ the {\em potential} outcome under treatment $A=a$, we want to estimate the average potential outcomes in the target population.
\begin{equation} \label{eq:apo_target}
    \mu_a \coloneqq \E[Y^a \mid S=0].
\end{equation}
A more common causal estimand, the average treatment effect, can be directly derived from the average potential outcomes ({\em e.g.},\ $\mu_1 - \mu_0$ in a binary treatment setting). Focusing on potential outcomes allows for simpler exposition, and they are of independent interest in many applications. 

The challenge in estimating $\mu_a$ is three-fold. The first is obvious: no outcome data is available for non-participants. One can contemplate resorting to outcome data from the trial, which brings us to the second challenge. The potential outcome $Y^a$ can only be observed for those who received treatment $A=a$. When treatment assignment depends on {\em unobserved} factors that also affect the outcome, one risks {\em confounding} bias, which presents a non-trivial challenge in analyzing {\em observational} data (see \Cref{sec:lppod}). However, it is easily avoided in trials by {\em randomized} treatment assignment, and the average potential outcome in {\em trial population}, $\E[Y^a \vert S=1]$, can be reliably estimated. The final challenge is $\E[Y^a \vert S=0] \neq \E[Y^a \vert S=1]$ when there is confounding by {\em trial participation}, leading to different distributions of effect modifiers in trial and target populations (see \Cref{fig:1}). Therefore, one cannot generalize population-level effect estimates from a trial to the target population, but needs to adjust for confounding by trial participation. 

Next, we state the causal assumptions needed to estimate $\mu_a$ by incorporating outcome data from the trial.
\subsection{Assumptions for Causal Inference}
\begin{assumption}[{\em Consistency}] \label{asm:cons}
            $A=a \implies Y=Y^a$.
\end{assumption}
\begin{assumption}[{\em Mean ignorability of treatment assignment in trial}] \label{asm:nuc}
        $\E[Y^a \mid X, S=1] = \E[Y^a \mid X, S=1, A=a]$.
\end{assumption}
\begin{assumption}[{\em Positivity of treatment assignment in trial}] \label{asm:pot}
             $P(A=a \mid X=x, S=1) > 0$.
\end{assumption}
Assumptions \ref{asm:cons}-\ref{asm:pot} are satisfied in an RCT by design, and they enable causal inference within the trial population, {\em i.e.},\ reliable estimation of $\E[Y^a \mid S=1]$.
\begin{assumption}[{\em Mean ignorability of trial participation}] \label{asm:nuem}
             $\E[Y^a \mid X] = \E[Y^a \mid X, S=1] = \E[Y^a \mid X, S=0]$.
\end{assumption}
\begin{assumption}[{\em Positivity of selection into trial}] \label{asm:pos}
            \qquad $P(S=0 \mid X=x) > 0 \implies P(S=1 \mid X=x) > 0$.
\end{assumption}
Assumption~\ref{asm:nuem} requires that within levels of measured covariates $X$, potential outcomes in the trial and target populations are the same on average. Assumption~\ref{asm:pos} ensures that every patient has a nonzero probability of participating in the trial, and we do not have to rely on pure {\em extrapolation}. Assumptions~\ref{asm:nuem} and \ref{asm:pos} transform the problem of \enquote{generalizing the results of a trial} into a {\em covariate} shift problem and allow one to identify $\mu_a$ as follows \cite{dahabreh2020extending}.
\begin{align} 
    \mu_a
    &= \E_{X \sim P_0} [ \E[ Y^a \mid X, S=0]] \nonumber \\
    &= \E_{X \sim P_0}[ \E[ Y^a \mid X, S=1]]  \nonumber \\
    &= \E_{X \sim P_0}[ \E[ Y \mid X, S=1, A=a]]. \label{eq:apo_iden}
\end{align}
where last two steps follow from Assumptions \ref{asm:nuem}-\ref{asm:pos} and \ref{asm:cons}-\ref{asm:pot}. Note that \eqref{eq:apo_iden} can be estimated using only covariate data from non-participants ($S=0$) and covariate, treatment, and outcome data from the trial participants ($S=1$).
\section{Generalization Using Experimental Data} \label{sec:LFTG}
\citet{dahabreh2020extending} propose estimators of \eqref{eq:apo_iden} based on outcome functions, weighting by the inverse of the trial participation probability, and doubly-robust (DR) ones. We focus on the outcome function approach as it more lucidly uncovers the limitations of generalization from trial data, and the synthetic results in \citet{dahabreh2020extending} show that it outperforms the weighting-based approaches and performs on par with the DR ones (as we also verify in \Cref{app:ipw_dr}). Our findings reveal how a predictive model trained on large-scale observational data could help.

We define the mean outcome function in the trial population $S=1$ under treatment $A=a$ as
\begin{align}
    g_a (X) 
    &\coloneqq \E[Y \mid X,S=1,A=a] \label{eq:mof} \\
    &= \E[Y^a \mid X,S=1] \tag{Assumptions~\ref{asm:cons}-\ref{asm:pot}} \\
    &= \E[Y^a \mid X]. \tag{Assumptions~\ref{asm:nuem}-\ref{asm:pos}}
\end{align}
One can estimate $\hat{g}_a (X)$ from the trial sample ${\cal D}_1$, and then average its predictions in the target sample ${\cal D}_0$. This leads to the following outcome model (OM) estimator on the composite sample ${\cal D}$.
\begin{equation} \label{eq:apo_est_rct}
    \hat{\mu}_{a}^{\tn{OM}} = \frac{1}{n_0} \sum_{i=1}^n \ind{S_i = 0} \hat{g}_a (X_i).
\end{equation}
In the remainder of this section, we investigate when $\hat{\mu}_{a}^{\tn{OM}}$ is expected to have high mean-squared error (MSE). Our next result gives an approximation for the MSE in the special case where $X$ is purely categorical, which provides perspective into the limitations of $\hat{\mu}_{a}^{\tn{OM}}$ for the more general case as well.

\begin{restatable}[]{proposition}{varLb}
    \label{proposition:varLb}
    Let $X$ be a categorical covariate stratifying the population into $K$ groups and denote by $n_{s=1,a,k}$ the number of trial participants from group $X=k$ assigned to treatment $A=a$, and by $\sigma^2_{a,k}$ the variance of outcome among such patients. Let us estimate the outcome function $g_a (X=k)$ with the sample mean of outcomes $Y$ of participants in group $X=k$ assigned to treatment $A=a$.
    
    Suppose that Assumptions~\ref{asm:cons}-\ref{asm:pos} hold. When $n_0$ is large, the MSE of $\hat{\mu}_{a}^{\tn{OM}}$ in \eqref{eq:apo_est_rct} can be approximated as
    \begin{equation} \label{eq:varLB}
        \E[(\hat{\mu}_{a}^{\tn{OM}} - \mu_a)^2] \approx \sum_{k=1}^K p^2_{s=0} (k)\frac{\sigma^2_{a,k}}{n_{s=1,a,k}},
    \end{equation}
    where $p_{s=0} (k) \coloneqq P(X=k \mid S=0)$ is the proportion of patients from group $X=k$ in the target population.
\end{restatable}

\Cref{proposition:varLb} reveals the key challenge in our endeavor. The reason behind the need for a \enquote{generalization procedure} is that some effect modifiers' distributions might differ in the trial and target populations. Reading off \eqref{eq:varLB}, one can see that when the trial is limited in representing patient profiles that are prevalent in the target population (small $n_{s=1, a,k}$, large $p_{s=0} (k)$), the MSE will be larger. That is, inference in target population gets more challenging when it becomes \enquote{more different} from the trial population.

The insights from \Cref{proposition:varLb} extend to the case with continuous covariates and parametric estimators $\hat{g}_a (X) = g_a (X; \hat{\theta})$ ({\em e.g.},\ a random forest). Let us denote by ${\cal A}$ the algorithm that fits $\hat{\theta}$ from the trial sample ({\em e.g.},\ ridge regression), {\em i.e.},\ $\hat{\theta} = {\cal A} ({\cal D}_1)$. As ${\cal D}_1 \sim P_1$, we write $\hat{\theta} \sim {\cal A} (P_1)$ to refer to the randomness in estimating $\hat{\theta}$ from ${\cal D}_1$. Next, we give an approximation for the MSE in the general case.

\begin{restatable}[]{theorem}{mseLem}
    \label{theorem:mseLem}
    Suppose that Assumptions~\ref{asm:cons}-\ref{asm:pos} hold and consider a parametric estimator $\hat{g}_a (X) = g_a(X;\hat{\theta})$ for the outcome function. For large $n_0$, the MSE of $\hat{\mu}_{a}^{\tn{OM}}$ in \eqref{eq:apo_est_rct} can be approximated as
    \begin{align} 
        &\E[(\hat{\mu}_{a}^{\tn{OM}} - \mu_a)^2] \nonumber \\
        &\approx \E_{X \sim P_0} \big[ \underbrace{\E_{\hat{\theta} \sim {\cal A} (P_1)}[ g_a (X; \hat{\theta})] - g_a (X)}_{\eqqcolon ~\tn{SB}_g (X)} \big]^2 \label{eq:mse_p1} \\
        &\hspace{10pt}+\V_{\hat{\theta} \sim {\cal A} (P_1)}\big( \E_{X \sim P_0} [ g_a (X; \hat{\theta}) ] \big). \label{eq:mse_p2}
    \end{align}
\end{restatable}

The first term, \eqref{eq:mse_p1}, is the statistical bias (SB). Crucially, it is obtained by integrating the bias function $\tn{SB}_g (X)$ over the {\em target} covariate distribution $P_{X \mid S=0}$, making $\hat{\mu}_{a}^{\tn{OM}}$ susceptible to weak overlap between trial and target populations. Consider the case where $g_a (X; \hat{\theta})$ is misspecified/underfit. $\tn{SB}_g (X)$ will be larger where $P_{X \mid S=1}$ has little weight, since $g_a (X; \hat{\theta})$ is fit using the {\em trial} sample ${\cal D}_1$. $\hat{\mu}_{a}^{\tn{OM}}$ may then suffer substantial bias if $P_{X \mid S=0}$ is large in covariate regions where the trial support is weak. It is therefore essential to ensure that $g_a (X; \hat{\theta})$ is rich enough and can match the complexity of $g_a (X)$ to avoid a large bias term. However, in practice, one's ability to flexibly model $g_a (X; \hat{\theta})$ is severely limited as the small size of trials ({\em e.g.,} $\sim 200$) can lead to overfitting, {\em i.e.},\ increasing the variance term \eqref{eq:mse_p2}. We empirically demonstrate this tradeoff in \Cref{app:bvt}.
\section{Prediction-powered Generalization Using Experimental and Observational Data} \label{sec:lppod}
Here, we study how integrating rich observational data with limited experimental data can make the generalization task more statistically feasible. We index by $S=2$ the observational population with joint distribution $P_2$. We assume access to an i.i.d. sample of observations $(X_i, A_i, Y_i) \sim P_{2}$ and denote by ${\cal D}_{2,a}$ the set of patients who received treatment $A=a$ in the observational data. In the first step, we fit a predictive model $f_a: {\cal X} \to \mathbb{R}$ by minimizing the empirical mean-squared error in ${\cal D}_{2,a}$ to approximate
\begin{equation} \label{eq:obs-om}
    \E[Y \mid X, S=2, A=a].
\end{equation}
Unlike the trial sample, large observational data can support parametrizing $f_a(X)$ with powerful machine learning models, allowing it to model complex functions.

If one is willing to make Assumptions~\ref{asm:cons}-\ref{asm:pos} for the observational study (OS), {\em i.e.}, $S=2$ instead of $S=1$, $\mu_a$ can be identified as $\E_{X \sim P_0} [\E [Y \lvert X, S=2, A=a]]$ via the same machinery in \eqref{eq:apo_iden}. One could then apply $f_a(X)$ in the composite sample ${\cal D}$ to estimate $\mu_a$ as 
\begin{equation} \label{eq:apo_est_os_naive}
    \hat{\mu}_{a}^{\tn{OS-OM}} = \frac{1}{n_0} \sum_{i=1}^n \ind{S_i = 0} f_a (X_i),
\end{equation}
analogous to \eqref{eq:apo_est_rct}. While Assumptions~\ref{asm:cons}-\ref{asm:pos} are defensible for trials, some of them rarely hold for observational studies in practice, particularly the ignorability of treatment assignment (no unmeasured confounding). In contrast to most of the literature on causal inference using observational data, we take an extremely assumption-light approach, making {\em no} assumptions on observational data. We define the \enquote{bias function} as the difference between the outcome function in the trial, $g_a (X)$, and the observational predictor, $f_a (X)$.
\begin{align}
    b_a (X) 
    &\coloneqq f_a (X) - g_a (X) \label{eq:bax_defn} \\
    &=\underbrace{f_a(X) - \E[Y \mid X, S=2, A=a]}_{\tn{statistical~bias}} \nonumber \\
    &\hspace{5pt}+\underbrace{\E[Y^a \mid X, S=2, A=a] - \E[Y^a \mid X, S=2]}_{\tn{confounding~bias}} \nonumber \\
    &\hspace{5pt}+\underbrace{\E[Y^a \mid X, S=2] - \E[Y^a \mid X, S=1]}_{\tn{transportation~bias}},  \nonumber
\end{align}
since $\E[Y^a \mid X, S=1] = g_a (X)$ (see \eqref{eq:mof}). The statistical bias term is related to fitting $f_a (X)$ using a finite sample, and it vanishes with more data given enough model capacity. Confounding and transportation biases, however, are the price of avoiding Assumptions \ref{asm:nuc} and \ref{asm:nuem} for the observational study ($S=2$). They will not disappear even with infinite data from the observational population, rendering $\hat{\mu}_{a}^{\tn{OS-OM}}$ in \eqref{eq:apo_est_os_naive} an {\em inconsistent} estimator for $\mu_a = \E [Y^a \mid S=0]$ even when $f_a(X) = \E[Y \mid X, S=2, A=a]$. 

In Sections~\ref{ssec:LEG} and \ref{ssec:padj}, we derive two new identifications of $\mu_a$ that integrate the predictions of $f_a$ in a statistically valid way and discuss how they lead to more sample-efficient estimation in comparison to \eqref{eq:apo_iden}. We give regression function-based estimators, derive their MSEs, and compare them to that of the baseline in \Cref{theorem:mseLem}. 
\subsection{Additive Bias Correction to Predictive Model} \label{ssec:LEG}
We covered why using $f_a$ alone is unreliable. Nonetheless, it may carry useful signal we can exploit when coupled with trial data. First, using the trial sample ${\cal D}_1$, we show how one can learn the bias function of the predictive model, $b_a (X) = f_a(X) - g_a(X)$. We then give an estimator for $\mu_a$ that uses the predictions $f_a (X)$ in the target sample by correcting with their estimated bias, $\hat{b}_a (X)$. We formalize in \Cref{theorem:ppommse} and \Cref{sssec:prr} when it is more advantageous to construct an estimator of $\mu_a$ that relies on fitting the bias function $b_a (X)$ instead of the outcome function $g_a(X)$, such as the illustrative example depicted in \Cref{fig:2}.
\subsubsection{Identification} \label{sssec:id-abc}
We start by trivially writing
\begin{equation} \label{eq:alt_iden_tapo}
        \mu_a = \E \left[ f_a(X) \mid S=0 \right] - \E \left[f_a(X) \! - \! Y^a \mid S=0 \right].
\end{equation}
The first term can be estimated by averaging $f_a(X)$ in the target sample, which is generally biased for $\mu_a$. The second term \enquote{removes} this bias; however, as it contains a {\em counterfactual} variable, $Y^a$, it is not immediately clear how one would estimate it. Our next result shows that it can be identified without additional assumptions on $f_a$.
\begin{figure}[t]
    \centering
    \includegraphics[width=\linewidth]{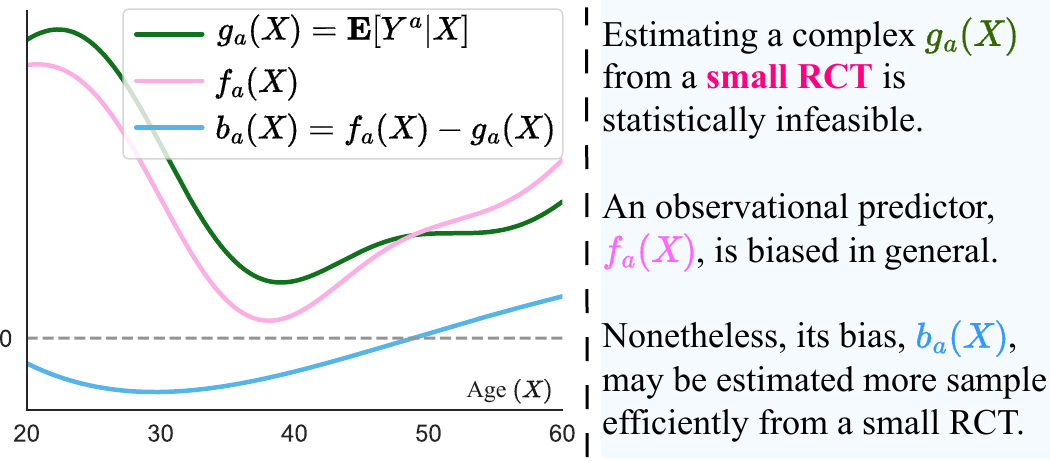}
    \vspace{-10pt}
    \caption{A biased predictor $f_a (X)$ can still capture higher order polynomials, making its bias $b_a (X)$ \enquote{easier} to learn than $g_a (X)$.}
    \label{fig:2}
\end{figure}

\begin{restatable}[]{lemma}{altiden}
    \label{lemma:alt_iden}
    Suppose that Assumptions~\ref{asm:cons}-\ref{asm:pos} hold. Let $f_a: {\cal X} \to \mathbb{R}$ and define the error variable
    \begin{equation} \label{eq:zerror}
        Z \coloneqq f_a(X) - Y.
    \end{equation}
    $\mu^a$ can be identified as
    \begin{equation}
        \mu_a = \E_{X \sim  P_0} [ f_a(X) - \E [Z \mid X, S=1, A=a] \big]. \label{eq:altidenlemma}
    \end{equation}
\end{restatable}

Note that $Z$ can be calculated for trial participants and second term in \eqref{eq:altidenlemma} can be estimated with covariate information from the target sample and covariate, treatment, and \enquote{error} information from the trial sample, as we cover next.
\subsubsection{Regression Function-based Estimation} \label{sssec:reg-est-abc}
By \eqref{eq:bax_defn}, \eqref{eq:zerror}, and \eqref{eq:mof}, it is straightforward to see that
\begin{equation*}
    b_a (X) = \E [Z \mid X, S=1, A=a].
\end{equation*}   
That is, the identification in \eqref{eq:altidenlemma} is through the bias function in \eqref{eq:bax_defn}. We denote by $b_a(X; \hat{\gamma})$ a parametric fit obtained by regressing $Z$ onto covariates $X$ in the trial sample and write the additive-bias-correction (ABC) estimator.
\begin{equation} \label{eq:ppom}
    \hat{\mu}_{a}^{\tn{ABC}} = \frac{1}{n_0} \sum_{i=1}^n \ind{S_i = 0} \big(f_a (X_i) - b_a(X_i; \hat{\gamma}) \big).
\end{equation}

\begin{restatable}[]{theorem}{ppommse}
    \label{theorem:ppommse}
    Suppose that Assumptions~\ref{asm:cons}-\ref{asm:pos} hold. For large $n_0$, the MSE of $\hat{\mu}_{a}^{\tn{ABC}}$ in \eqref{eq:ppom} can be approximated as
    \begin{align} 
        &\E[(\hat{\mu}_{a}^{\tn{ABC}} - \mu_a)^2] \nonumber \\
        &\approx \E_{X \sim P_0} \big[ \E_{\hat{\gamma} \sim {\cal A} (P_1)}[ b_a (X; \hat{\gamma})] - b_a (X) \big]^2  \label{eq:ppbm_mse_p1} \\
        &\hspace{10pt}+\V_{\hat{\gamma} \sim {\cal A} (P_1)}\big( \E_{X \sim P_0} [ b_a (X; \hat{\gamma}) ] \big). \label{eq:ppbm_mse_p2}
    \end{align}
\end{restatable}

\begin{figure}[t]
\vspace{-0.7em}
\begin{algorithm}[H]
   \caption{Generalization via additive bias correction}
   \label{alg:abc}
\begin{algorithmic}
   \STATE {\bfseries Input:} Sample ${\cal D}$, Predictor $f_a$, MSE optimizer ${\cal A}$
   \STATE ${\cal D}_{1,a} \subset {\cal D}$: Trial cohort ($S_i = 1$) with treatment $A_i=a$
   \FOR{$W_i \in {\cal D}_{1,a}$} 
   \STATE Calculate the prediction error $Z_i = f_a (X_i) - Y_i$
   \ENDFOR
   \STATE Fit $b_a (X; \hat{\gamma})$ by minimizing MSE for $Z$ in ${\cal D}_{1,a}$ using ${\cal A}$
   \STATE Return $\hat{\mu}_{a}^{\tn{ABC}}$ in \eqref{eq:ppom}
\end{algorithmic}
\end{algorithm}
\vspace{-2em}
\end{figure}
The significance of \Cref{theorem:mseLem} is showing that the MSE of $\hat{\mu}_{a}^{\tn{ABC}}$ admits the same form with that of $\hat{\mu}_{a}^{\tn{OM}}$ in \Cref{theorem:mseLem}. The difference is that the MSE is governed by how well the bias function $b_a(X)$ is estimated instead of the outcome function $g_a(X)$. This result formalizes how leveraging a potentially biased observational predictor can be more viable for the \enquote{generalization} task. Consider the case in \Cref{fig:2} where $f_a(X)$ captures higher degree polynomials in $g_a(X)$, resulting in $b_a (X)$ being a low-degree polynomial. One can then fit a linear model with a few polynomial features for $b_a (X; \hat{\gamma})$, resulting in both controlled bias \eqref{eq:ppbm_mse_p1} and variance \eqref{eq:ppbm_mse_p2} terms. On the other hand, fitting $g_a (X; \hat{\theta})$ similarly will result in a large bias term \eqref{eq:mse_p1}. We provide a detailed discussion in the next section and empirically demonstrate how the bias (\eqref{eq:mse_p1}, \eqref{eq:ppbm_mse_p1}) and variance (\eqref{eq:mse_p2}, \eqref{eq:ppbm_mse_p2}) terms compare in \Cref{app:bvt}.
\subsubsection{Case Study: Polynomial Ridge Regression} \label{sssec:prr}
The symmetry between the MSEs in Theorems~\ref{theorem:mseLem} and \ref{theorem:ppommse} allows one to reason about the (comparative) performances of the outcome and bias function-based estimators in \eqref{eq:apo_est_rct} and \eqref{eq:ppom}. To gain further insight, we study the polynomial ridge regression framework and describe the regime where estimating $b_a (X)$ is more feasible than $g_a (X)$ in the setting described below. We focus on {\em finite-sample} results, which are of significant interest given the limited size of trials.

We consider $X \in [-1,1]$, denote by $L^2([-1,1])$ the space of square-integrable functions \footnote{$\int_{-1}^1 f^2(x) \diff x < \infty, \quad \forall f \in L^2([-1,1]).$} endowed with the inner-product $\langle f,g \rangle = \int_{-1}^1 f(x) g(x) \diff x$, and assume that $g_a, b_a \in L^2([-1,1])$ with bounded norms $\norm{g_a}, \norm{b_a} \leq 1$. Finally, we assume the following generative equations.
\begin{align}
    Y_i = g_a (X_i) + \eta_i, \qquad Z_i = b_a (X_i) - \eta_i \label{eq:gen1z},
\end{align}
where $\eta_i \sim {\cal N} (0, \sigma^2)$ are zero-mean i.i.d. noise variables and $Z_i$ are the patient-wise {\em error} terms for the predictive model $f_a(X)$, defined previously in \eqref{eq:zerror}. 

Let us now define, for a generic function $f$, the \enquote{empirical excess risk} of a fit $\hat{f}$ obtained from a sample of size $m$ as
\begin{equation} \label{eq:excess_risk}
    R_m (\hat{f}, f) \coloneqq \frac{1}{m} \sum_{i=1}^m ( \hat{f} (X_i) - f(X_i) )^2,
\end{equation}
which quantifies how far away the fit is from the true function. In the rest of this section, we study oracle upper bounds on $R_{n_1} (\hat{g}_a, g_a)$ and $R_{n_1} (\hat{b}_a, b_a)$ when $\hat{g}_a$ and $\hat{b}_a$ are fit via polynomial ridge regression in the trial sample ${\cal D}_1$. To that end, let us now introduce the Legendre polynomials which have convenient properties that facilitate clear exposition. We denote by $\phi_k : [-1,1] \to \mathbb{R}$ the $k$-th order {\em normalized} Legendre polynomial. The set $\{ \phi_k \}_{k=0}^{\infty}$ form an {\em orthonormal basis} \footnote{$\langle \phi_i, \phi_j \rangle = \delta_{ij}$, $\text{span}(\{ \phi_k \}_{k=0}^{\infty}) = L^2([-1,1])$.} for $L^2([-1,1])$; meaning that any function $f \in L^2([-1,1])$ can be uniquely represented as a linear combination of $\{ \phi_k \}_{k=0}^{\infty}$, allowing us to write
\begin{align} \label{eq:leg_exp}
    g_a (X) = \sum_{k=0}^{\infty} \lambda_k \phi_k (X), \quad b_a (X) = \sum_{k=0}^{\infty} \omega_k \phi_k (X),
\end{align}
where $\lambda_k = \langle g_a, \phi_k \rangle$ and $\omega_k = \langle b_a, \phi_k \rangle$. In practice, one can fit $\hat{g}_a$ and $\hat{b}_a$ using Legendre polynomials up to degree $d' \in \mathbb{N}$ with ridge regularization to avoid overfitting to the trial sample. Leaving the intermediary steps to \Cref{app:prr}, we proceed to state the corresponding upper bounds on the expected empirical excess risks. 

\begin{restatable}[Adopted from \citet{wainwright2019high}]{lemma}{lemRisk}
    \label{lemma:lemRisk}
    Let $X \in [-1,1]$, $g_a, b_a \!\in\! L^2([-1,1])$, $\norm{g_a}, \norm{b_a} \leq 1$, and consider the generative equations in \eqref{eq:gen1z} with noise variance $\sigma^2$. Denote by $\hat{g}_a$ and $\hat{b}_a$ the fits obtained by regressing $Y$ and $Z$ (see \eqref{eq:zerror}) onto $\{ \phi_k (X) \}_{k=0}^{d'}$ in trial sample ${\cal D}_1$ with an appropriately chosen ridge regularization penalty. We have
    \begin{align}
        \E_{X_i, \eta_i} [ R_{n_1} (\hat{g}_a, g_a) ] &\leq \sigma^2 d' / n_1  + \sum\nolimits_{k=d' + 1}^{\infty} \lambda_k^2, \label{eq:er_g} \\
        \E_{X_i, \eta_i} [ R_{n_1} (\hat{b}_a, b_a) ] &\leq \sigma^2 d' / n_1 + \sum\nolimits_{k=d' + 1}^{\infty} \omega_k^2.  \label{eq:er_b}
    \end{align}
\end{restatable}

The upper bounds in \eqref{eq:er_g} and \eqref{eq:er_b} share the first statistical error term, which grows with the number of polynomial features $d'$. Comparing the second terms reveals that estimating the bias function is favorable when $\sum_{k=d' + 1}^{\infty} \omega_k^2 < \sum_{k=d' + 1}^{\infty} \lambda_k^2$. One would expect the preceding condition to hold in two scenarios, which we discuss next. 

The first one is when the observational predictor is high quality. Precisely, if $f_a(X) \approx g_a (X)$, then $b_a(X) \approx 0$, implying small values for $\omega_k$ and sum of their squares. This is the same condition in \citet{angelopoulos2023prediction} for a black-box predictor to power better inference when coupled with a small amount of gold-standard data. In the context of {\em causal} inference, it would take the individual terms in \eqref{eq:bax_defn} to be as small as possible to warrant $f_a(X) \approx g_a (X)$, which requires observational study to have negligible hidden confounding for treatment assignment and to be transportable conditioned on $X$. 

The second scenario is when $b_a$ \enquote{mostly} consists of lower degree polynomials, that is, $w_k \approx 0$ for $k > d'$. This is a relaxed and more general version of the key assumption in \citet{kallus2018removing}, which requires $b_a (X)$ to be linear in $X$. The idea is that even when $f_a(X)$ is biased, it can still capture complex structure, such as the higher order polynomials modeling rapid turns in $g_a(X)$, and make $b_a (X)$ considerably simpler, as illustrated in \Cref{fig:2}.
\subsection{Augmented Outcome Modeling} \label{ssec:padj}
Here we draw from \citet{schuler2021increasing, liao2023transfer} and leverage the observational model by using its predictions as an additional regressor while estimating the outcome function from the trial. 

Using $f_a(X)$ as a covariate still makes for an easier estimation task when $g_a (X) = f_a (X) + b_a(X)$ with $f_a (X)$ capturing most of the complexity in $g_a(X)$ and $b_a(X)$ is a simpler function. However, it has two advantages over the additive bias correction approach we discuss below.

First is \textbf{robustness} when $f_a (X)$ does not carry useful information. For instance, let $f_a(X)$ be an independent noise term, $\eta$, for all $X$. Then the additive bias $b_a (X) = \eta - g_a (X)$ is just a noisier version of $g_a (X)$ and even more challenging to estimate. On the other hand, when the predictions $f_a(X)$ are used as a covariate, a good learning algorithm would just ignore it. We compare the two approaches' robustness with synthetic experiments (see \Cref{fig:faxn}).

Second is the \textbf{flexibility} in how the predictions are utilized. Consider the illustrative example where
$
    g_a(X) = f_a(X) / 2
$
and the bias of $f_a(X)$ can be corrected simply dividing it by two. On the other hand, additive bias $b_a(X) = g_a(X)$ is identical to the outcome function and not easier to estimate.
\subsubsection{Identification} \label{sssec:id-pa}
Let us define the {\em augmented} covariate vector as
\begin{equation} \label{eq:adcov}
    \tilde{X}_i \coloneqq [X_i^1, X_i^2, \ldots, X_i^d, f_a(X_i)],
\end{equation}
where $X_i^n$ is the $n$-th {\em original} covariate out of $d$. We denote $\tilde{X} \in \tilde{{\cal X}}$ where $\tilde{{\cal X}} = {\cal X} \times \mathbb{R}$. 

\begin{restatable}[]{lemma}{altiden2}
    \label{lemma:alt_iden_2}
    Suppose that Assumptions~\ref{asm:cons}-\ref{asm:pos} hold. Let $f_a: {\cal X} \to \mathbb{R}$ and define the augmented outcome function
    \begin{equation} \label{eq:hax_defn}
        h_a (\tilde{X}) \coloneqq \E[ Y \mid \tilde{X}, S=1, A=a].
    \end{equation} 
    where $\tilde{X}$ is defined in \eqref{eq:adcov}. $\mu^a$ can be identified as
    \begin{equation}
        \mu_a = \E_{X \sim P_0}[ h_a (\tilde{X}) ]. \label{eq:apo_iden_3}
    \end{equation} 
\end{restatable}

Note that $X$ and $\tilde{X}$ carry the same {\em information} and Assumptions \ref{asm:cons}-\ref{asm:pos} continue to hold for $\tilde{X}$. The identification in \eqref{eq:apo_iden_3} thus follows from the same steps that lead to \eqref{eq:apo_iden}.

\begin{figure}[t]
\begin{algorithm}[H]
   \caption{Generalization via augmented outcome model}
   \label{alg:aom}
\begin{algorithmic}
   \STATE {\bfseries Input:} Sample ${\cal D}$, Predictor $f_a$, MSE optimizer ${\cal A}$
   \STATE ${\cal D}_{1,a} \subset {\cal D}$: Trial cohort ($S_i = 1$) with treatment $A_i=a$
   \FOR{$W_i \in {\cal D}$}   
   \STATE Calculate the outcome prediction $f_a (X_i)$
   \STATE Construct the augmented covariate vector $\tilde{X}_i$ as in \eqref{eq:adcov}
   \ENDFOR
   \STATE Fit $h_a (\tilde{X}; \hat{\beta})$ by minimizing MSE for $Y$ in ${\cal D}_{1,a}$ using ${\cal A}$
   \STATE Return $\hat{\mu}_{a}^{\tn{AOM}}$ in \eqref{eq:apo_est_com}
\end{algorithmic}
\end{algorithm}
\vspace{-1em}
\end{figure}
\subsubsection{Regression Function-based Estimation} \label{sssec:reg-est-pa}
Note that $h_a (\tilde{X}) = h_a (X, f_a(X)) = g_a(X)$ and the only difference from the baseline approach in \Cref{sec:LFTG} is that we have an additional regressor, $\tilde{X}_{d+1} = f_a(X)$. We denote by $h_a (\tilde{X}_i; \hat{\beta})$ the parametric fit for $h_a (\tilde{X}_i)$ and write the {\em augmented} outcome modeling (AOM) estimator as
\begin{equation} \label{eq:apo_est_com}
    \hat{\mu}_{a}^{\tn{AOM}} = \frac{1}{n_0} \sum_{i=1}^n \ind{S_i = 0} h_a (\tilde{X}_i; \hat{\beta}).
\end{equation}
The approximation to the MSE $\E[ (\hat{\mu}_a^{\textnormal{AOM}} -\mu_a)^2]$ follows the same form with that of $\hat{\mu}_a^{\textnormal{OM}}$ in \Cref{theorem:mseLem} with the augmented outcome function $h_a$ replacing $g_a$. In the interest of space, we defer the precise statement to \Cref{app:aomproof}.

We close this section by mentioning two critical directions for future work to leverage observational data more efficiently: one related to {\em modeling} and the other to {\em estimation}.

\paragraph{Representation-powered Outcome Modeling} 
Instead of using a model's predictions $f_a (X)$, one can use the representations learned by the model as additional covariates in the trial \cite{johansson2016learning, shalit2017estimating}. This approach also allows for extracting richer information from the observational data more flexibly, {\em e.g.}, via unsupervised and multimodal learning methods.  

\paragraph{Doubly-robust Estimation} 
For the prediction-powered identifications of $\mu_a$ in \eqref{eq:altidenlemma} and \eqref{eq:apo_iden_3}, we focused only on regression function-based estimators to demonstrate the advantages of our approach. In \Cref{app:drest}, we give doubly-robust estimators that enjoy desirable properties such as asymptotic normality that enable the construction of confidence intervals \cite{chernozhukov2018double, kennedy2023towards}.
\section{Synthetic Experiments} \label{sec:synth}
We simulate over a thousand different synthetic data generating processes with varying levels of complexity in the outcome function $g_a(X)$, confounding bias in the observational study, and trial size $n_1$. We compare the root MSE (RMSE) of our estimators \eqref{eq:ppom} and \eqref{eq:apo_est_com}, which combine experimental and observational data, to that of the baselines \eqref{eq:apo_est_rct} and \eqref{eq:apo_est_os_naive} which use them alone. Bias-variance terms ({\em e.g.,} \eqref{eq:ppbm_mse_p1} and \eqref{eq:ppbm_mse_p2}) are presented in \Cref{app:bvt}. Further, we demonstrate the robustness of the augmented outcome modeling estimator in \eqref{eq:apo_est_com} over the additive bias correction estimator in \eqref{eq:ppom}. While the main results are concerned with outcome-modeling-based estimators, we present additional empirical results for the inverse propensity weighting and and doubly-robust estimators in \Cref{app:ipw_dr}. Our code is available at \url{https://github.com/demireal/ppci}.
\begin{figure}[t]
    \centering
    \includegraphics[width=\linewidth]{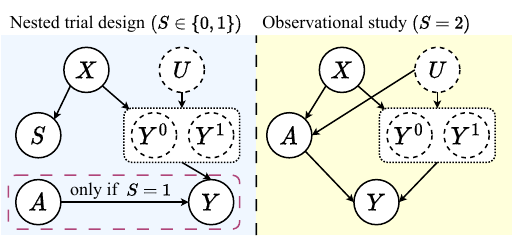}
    \vspace{-10pt}
    \caption{Data-generating process used in simulated experiments. ({\em Left.}) $X$ (observed) induces confounding by trial participation. ({\em Right.}) In the observational study, there is {\em hidden} confounding for treatment assignment due to $U$ (unobserved).}
    \label{fig:dgp}
\end{figure}
\subsection{Data-generating Process}
We consider two covariates $X, U \in [-1,1]$, a binary treatment strategy $A \in \{0, 1\}$, and a real-valued outcome $Y \in \mathbb{R}$. We first describe the probabilistic model that generates the potential outcomes $Y^0$ and $Y^1$ conditioned on $X$ and $U$. We then move on to explain the sampling mechanism in the nested trial design that generates the trial and target cohorts. Finally, we specify the patient sampling and treatment assignment mechanism underpinning the observational data we use to train a predictive model $f_a:[-1,1]^2 \to \mathbb{R}$. Results of simpler experiments where the functions underlying the data-generating process are specified to be linear are presented in \Cref{app:glmexp}.
\begin{figure*}[ht]
    \centering
    \includegraphics[width=\linewidth]{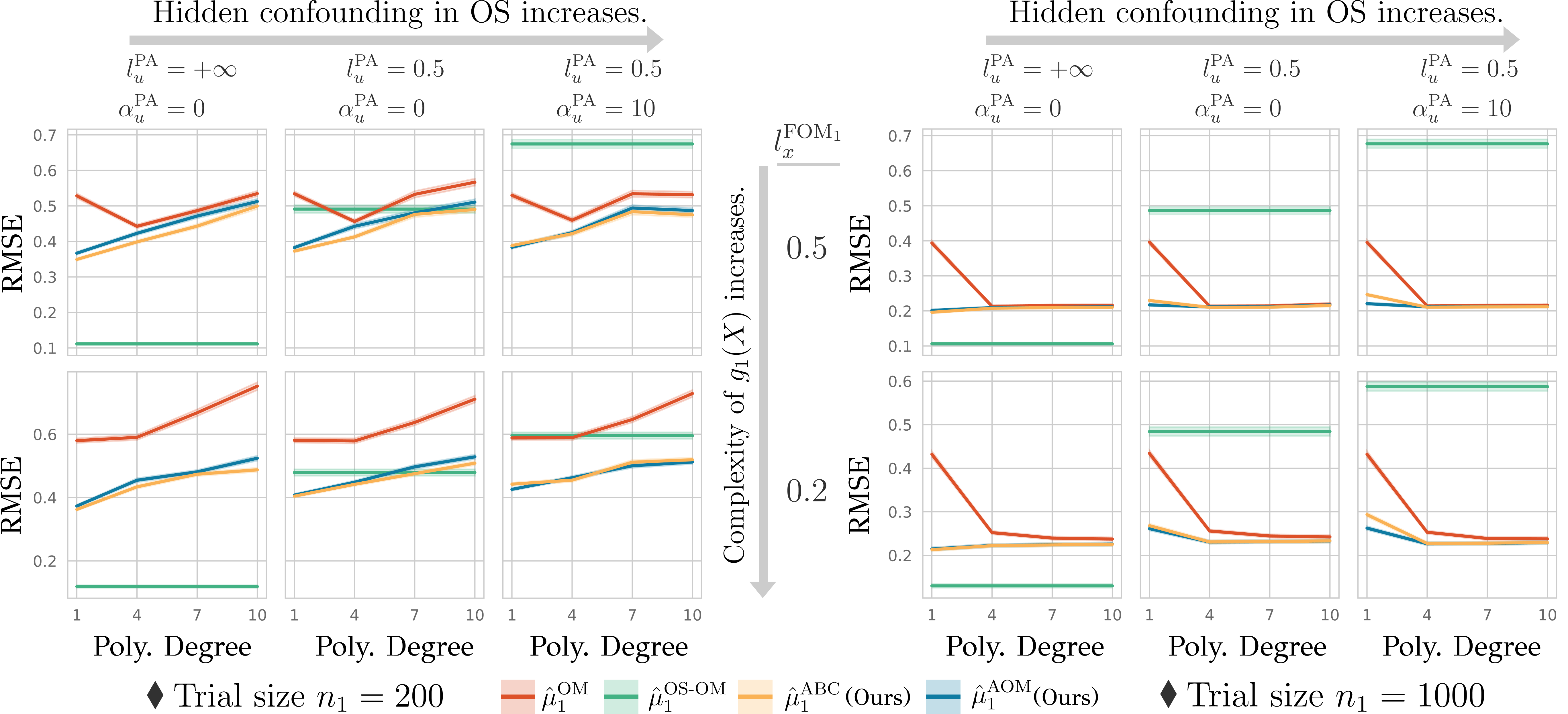}
    \vspace{-10pt}
    \caption{100 different set of data-generating functions are sampled for each $(l_x^{\tn{FOM}_1}, \alpha_u^{\tn{PA}}, n_1)$. We plot the RMSE averaged over 100 scenarios. Results are reported for four different numbers of polynomial features used to fit the underlying regression functions (if any).}
    \label{fig:synthetic-rmse}
\end{figure*}

\paragraph{Generating (Potential) Outcomes} 
We denote the {\em full} outcome model (FOM) for treatment $A=a$ by $\tn{FOM}_a: [-1,1]^2 \to \mathbb{R}$. For a patient with covariates $(X_i, U_i)$, the potential outcome is calculated as $Y^a_i = \tn{FOM}_a (X_i, U_i)$. Note that we use the same outcome model for patients in the trial, target, and observational samples.

We generate $\tn{FOM}_a$ by sampling from a GP with mean function $m(X,U) = 0$ and kernel function $k\left((X,U), (X',U')\right)$ \cite{rasmussen2006gaussian}. We create a composite kernel by adding a squared-exponential (SE) kernel to model the {\em local} variations and a linear kernel to model the trends in the outcome. Precisely, we have
\begin{align}
    &k((X,U),\!(X',U')) \!= \alpha_x^{\textnormal{FOM}_a} X X' + \alpha_u^{\textnormal{FOM}_a} U U' \tag{linear} \\
    &\hspace{5pt}+ \exp \left( -\frac{(X-X')^2}{2 (l_x^{\textnormal{FOM}_a})^2} - \frac{(U-U')^2 }{2 (l_u^{\textnormal{FOM}_a})^2} \right), \qquad \tn{(SE)} \label{eq:kernel_eq}
\end{align}

where $\alpha_x^{\textnormal{FOM}_a}, \alpha_u^{\textnormal{FOM}_a}, l_x^{\textnormal{FOM}_a}, l_u^{\textnormal{FOM}_a} \in \mathbb{R}_+$ are free parameters. We experiment with different values to simulate a diverse set of scenarios. For instance, a larger value for $\alpha_u^{\textnormal{FOM}_a}$ implies a stronger linear trend in $\tn{FOM}_a (X,U)$ along $U$-axis. More details are given at the end of this section. 

\paragraph{Generating Trial and Target Samples} 
We consider a {\em nested} study design and generate a composite {\em trial-eligible} patient cohort by sampling $X_i, U_i \sim \texttt{Uniform} [-1,1]$ independently. We denote by $P(S=1 \mid X_i, U_i)$ the probability of trial participation, which is generated as 
\begin{equation} \label{eq:psx_gp}
    P(S=1 \vert X_i, U_i) = \tn{median} \Big\{ \frac{1}{1 +  e^{-\tn{L}_{\tn{PS}}(X_i, U_i)}}, 0.1, 0.9 \Big\}.
\end{equation}
where the \enquote{logit} function $\tn{L}_{\tn{PS}}(X_i, U_i)$ is sampled from a GP with the composite linear + SE kernel in \eqref{eq:kernel_eq} with parameters $\alpha_x^{\textnormal{PS}} = 10$, $l_x^{\textnormal{PS}}=1$, $\alpha_u^{\textnormal{PS}} = 0$, $l_u^{\textnormal{PS}}=+\infty$. The last two parameters effectively imply that the trial participation probability does not depend on $U$ but $X$ only, ensuring \Cref{asm:nuem}. Taking a median with $0.1$ and $0.9$ ensures \Cref{asm:pos}. Trial participation is then sampled as $\texttt{Bernoulli} (P(S=1 \mid X_i, U_i))$.

Finally, for trial participants ($S_i=1$), the treatment assignment is sampled as $A_i \sim \texttt{Bernoulli} (0.5)$ and the observed outcome is generated as $Y = \tn{FOM}_{A_i} (X_i, U_i)$, which ensures that Assumptions~\ref{asm:cons}-\ref{asm:pot} hold.

\paragraph{Generating an Observational Sample} 
An observational cohort is generated by sampling $X_i, U_i \sim \tn{Uniform} [-1,1]$ independently. For each patient, treatments are sampled as $A_i \sim \texttt{Bernoulli} (P(A=1 \lvert S=2, X_i, U_i))$, where the probability of treatment assignment is generated similar to \eqref{eq:psx_gp} through a logit function $\tn{L}_{\tn{PA}}(X,U)$ sampled from a GP with parameters $\alpha_x^{\textnormal{PA}}, \alpha_u^{\textnormal{PA}}, l_x^{\textnormal{PA}}, l_u^{\textnormal{PA}} \in \mathbb{R}_+$. The observed outcomes are generated as $Y = \tn{FOM}_{A_i} (X_i, U_i)$. 

\paragraph{Simulated Scenarios and GP Parameters}
We focus on the mean potential outcome under treatment $A=1$ in the target population, $\mu_1 = \E [Y^1 \mid S=0]$. The sample size for the observational study (OS) and the target sample are set to $50,000$ and $20,000$, respectively. We experiment with different values for the parameters $n_1, l_x^{\tn{FOM}_1}$, and $\alpha_u^{\tn{PA}}$. We fit $f_1 (X)$ from the OS with a neural network, and $g_1 (X; \hat{\theta})$, $b_1 (X; \hat{\gamma})$, $h_1 (\tilde{X}; \hat{\beta})$ are fit from the trial sample ${\cal D}_1$ via polynomial ridge regression with 5-fold cross-validation.

We use trial sizes $n_1 \in \{200, 1000\}$ and $l_x^{\tn{FOM}_1} \in \{ 0.5, 0.2 \}$, where a smaller value leads $\tn{FOM}_1 (X,U)$ to change more quickly in response to $X$ ({\em i.e.}, consist of high order polynomials), thus resulting in a more complex outcome function $g_1 (X)$. We provide examples in \cref{app:ec}.
 
When learning $f_1(X)$ from the OS, we conceal $U$ and experiment with $(l_u^{\tn{PA}}, \alpha_u^{\tn{PA}}) \in \{ (\infty, 0), (0.5,0), (0.5,10) \}$. In the first setting, $P(A=1 \vert S=2, X, U)$ does not depend on $U$, and there is {\em no} hidden confounding, which is introduced when $l_u^{\tn{PA}}$ is changed to $0.5$. Finally setting $\alpha_u^{\tn{PA}}$ to $10$ increases the \enquote{weight} of $U$ in $P(A=1 \vert S=2,X,U)$, leading to a larger confounding bias.

The preceding sets of hyperparameters lead to $2 \times 2 \times 3 = 12$ combinations. For each combination, 100 different data-generating functions were sampled from the GPs, leading to 1200 distinct scenarios. For each scenario, 100 independent runs were made where a new trial sample ${\cal D}_1$ was generated, and estimates for $\mu_1$ were calculated. An average RMSE is calculated for each scenario over 100 runs and for each combination over 100 scenarios, presented in \Cref{fig:synthetic-rmse}.
\subsection{Discussion of Results}
We first discuss the results in Figure~\ref{fig:1} and focus on the advantages of our prediction-powered estimators over the baselines. We then investigate Figure~\ref{fig:2} which demonstrates why the augmented outcome modeling (AOM) is more robust than the additive bias correction (ABC) approach.

\paragraph{Using the OS Alone}
As the hidden confounding in the OS increases, $\hat{\mu}_1^{\tn{OS-OM}}$ in \eqref{eq:apo_est_os_naive}, which directly applies $f_1 (X)$ in the target population, suffers higher RMSE. Note that its performance does not improve with a larger trial size, as confounding bias does not result from a small sample size.

\paragraph{Using the Trial Sample Alone}
The performance of $\hat{\mu}_1^{\tn{OM}}$ in \eqref{eq:apo_est_rct} relies on fitting the outcome function $g_1 (X)$ from the trial sample accurately. Therefore, it incurs higher RMSE when the trial is small and $g_1 (X)$ is complex. The RMSE gets even worse when higher-order polynomials are used to fit $\hat{g}_1 (X)$ in a small trial, due to the quickly increasing variance term \eqref{eq:mse_p2} (see \Cref{app:bvt}).

\paragraph{Combining Trial and Observational Data}
Prediction-powered estimators $\hat{\mu}_{1}^{\tn{ABC}}$ and $\hat{\mu}_{1}^{\tn{AOM}}$ yield significant improvement over $\hat{\mu}_1^{\tn{OM}}$ when the trial is small, outcome function is complex, and the hidden confounding is small. This is because when $f_1 (X)$ accurately {\em estimates away} most of the complex structure in $g_1(X)$, the resulting bias function has a small norm, {\em i.e.}, $b_1 (X) \approx 0$ (see \Cref{app:ec} for some examples), and is more feasible to fit from the small trial sample and generalize to the target population.

Confounding in the OS leads to slightly worse RMSEs for $\hat{\mu}_{1}^{\tn{ABC}}$ and $\hat{\mu}_{1}^{\tn{AOM}}$, but they still compare favorably to $\hat{\mu}_1^{\tn{OM}}$. Note that a larger hidden confounding need not impede benefiting from $f_1(X)$, so long as $b_1 (X)$ consists of lower degree polynomials. The results in \Cref{fig:synthetic-rmse} are {\em averaged} over many scenarios, and we provide examples in \Cref{app:ec} where $b_1 (X)$ is also \enquote{complex} and no improvement over the baseline is achieved. Finally, when the trial is large enough to support fitting $g_1 (X)$ with a higher order polynomial, $\hat{\mu}_1^{\tn{OM}}$ perform on par with $\hat{\mu}_{1}^{\tn{ABC}}$ and $\hat{\mu}_{1}^{\tn{AOM}}$.

\paragraph{Which Prediction-powered Estimator is Better?}
For the experiments in \Cref{fig:synthetic-rmse}$, \hat{\mu}_{1}^{\tn{ABC}}$ and $\hat{\mu}_{1}^{\tn{AOM}}$ have similar performances. We demonstrate the robustness of $\hat{\mu}_{1}^{\tn{AOM}}$ in \Cref{fig:faxn}, where $f_1 (X)$ is just noise and $b_1 (X)$ is harder to estimate than $g_1 (X)$. While $\hat{\mu}_{1}^{\tn{ABC}}$ suffers high RMSE, $\hat{\mu}_{1}^{\tn{AOM}}$ simply ignores $f_1 (X)$ as a regressor and retains the performance of the baseline approach. 
\section{Concluding Remarks}
We investigated the statistical challenges of generalizing causal inferences from a randomized controlled trial to a target population whose characteristics differ from the trial. We showed how observational data could make generalization more statistically feasible without unrealistic assumptions. Through a diverse set of synthetic experiments, we verified the effectiveness of our methods. Future work includes investigating more flexible approaches to leverage observational data and exploring further experiment setups ({\em e.g.}, with different kernels to simulate specific real-world settings) to gain further insight into the potentials and limitations of integrating experimental and observational evidence.
\begin{figure}[t]
    \centering
    \includegraphics[width=\linewidth]{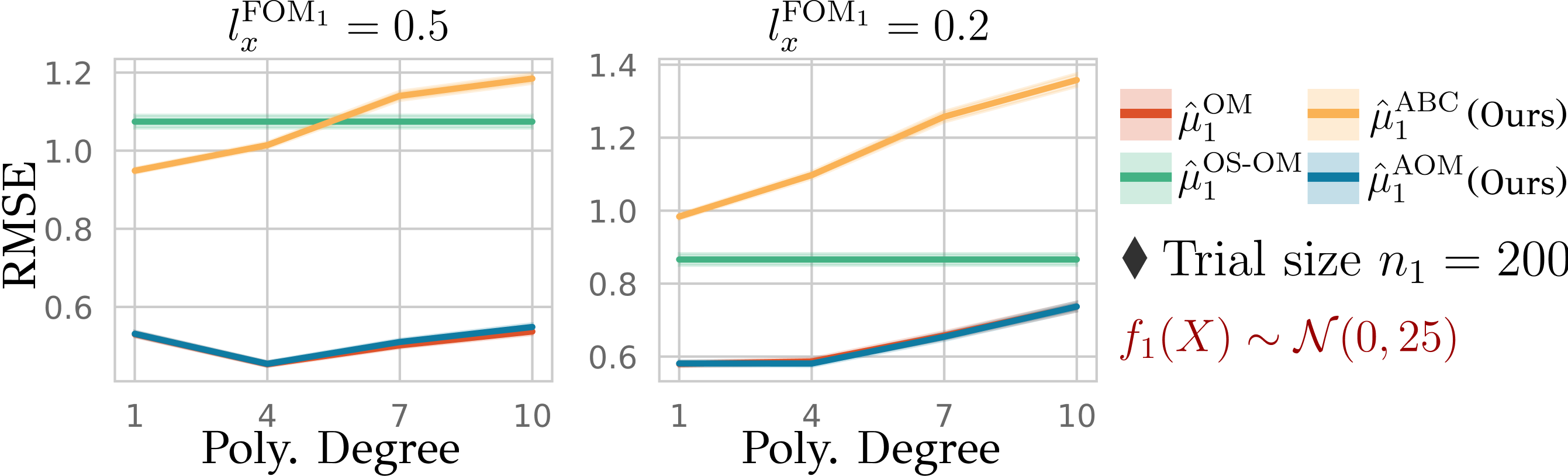}
    \vspace{-10pt}
    \caption{Convention same as \Cref{fig:synthetic-rmse}. The observational predictor is not trained on any data but generates i.i.d. noise for all $X$.}
    \label{fig:faxn}
    \vspace{-10pt}
\end{figure}

\section*{Acknowledgements}
The authors thank the anonymous reviewers for their thoughtful comments and contributions to our experimental results during the discussion phase, and to Sontag Lab members for insightful discussions. ID was supported by funding from the Eric and Wendy Schmidt Center at the Broad Institute of MIT and Harvard. This study was supported in part by Office of Naval Research Award No. N00014-21-1-2807. 
\section*{Impact Statement}
Causal inference is crucial in medicine. The present manuscript contributes to the rapidly growing field of integrating gold-standard data from trials and real-world evidence. We propose statistically valid methods that can leverage large-scale observational data to power better generalization of causal effects from a trial to a target population. There are many potential societal consequences of our work, none which we feel must be specifically highlighted here.

\bibliography{ref.bib}
\bibliographystyle{icml2024}

\newpage
\appendix
\onecolumn
\section{Proofs and Additional Results}
\subsection{Deferred Proofs}
\varLb*
\begin{proof}
Let us denote by ${\cal D}_{s=1,a,k} \subseteq {\cal D}_1$ the set of trial participants in group $k$ that received treatment $A=a$, with size $n_{s=1,a,k} = \sum_{i=1}^n \ind{X_i=k, S_i = 1, A_i=a}$. We estimate the outcome model as
\begin{equation} \label{eq:group_sample_mean}
    \hat{g}_a (X = k) = \frac{\sum_{i=1}^n \ind{X_i=k, S_i = 1, A_i=a} Y_i}{n_{s=1,a,k}}.
\end{equation}
$\hat{g}_a (X=k)$ is simply the sample mean of outcomes $Y_i$ in ${\cal D}_{s=1,a,k}$ and we have
\begin{align}
    \E \left[ \hat{g}_a (X=k) \right] 
    &= \E \left[ Y \mid X=k,S=1,A=a \right] \nonumber \\
    &= \E \left[ Y^a \mid X=k,S=1,A=a \right] \nonumber\\
    &= \E \left[ Y^a \mid X=k,S=1 \right] \nonumber \\
    &= \E \left[ Y^a \mid X=k,S=0 \right] \label{eq:lem_varLb_1}
\end{align}
where the first equality follows from the unbiasedness of the sample mean and the rest from Assumptions \ref{asm:cons}, \ref{asm:nuc}, and \ref{asm:nuem}, respectively.
When $X$ is categorical and $\hat{g}_a (X)$ is estimated via \eqref{eq:group_sample_mean}, \eqref{eq:apo_est_rct} admits the following equivalent expression.
\begin{equation} \label{eq:apo_rct_est_categ}
    \hat{\mu}_{a}^{\tn{OM}} = \sum_{k=1}^K \hat{p}_{s=0} (k) \hat{g}_a (k).
\end{equation}
where $\hat{p}_{s=0}(k) = \frac{\sum_{i=1}^n \ind{S_i = 0, X_i = k}}{\sum_{i=1}^n \ind{S_i = 0}}$ is the proportion of patients in the target sample ${\cal D}_0$ from group $k$. 

Note that the target ${\cal D}_0$ and trial ${\cal D}_1$ samples are disjoint of each other. Since $\hat{p}_{s=0} (k)$ is effectively calculated from the observations in the target sample ${\cal D}_0$ only, and  similarly $\hat{g}_a (k)$ from ${\cal D}_1$ only, $\hat{p}_{s=0} (k)$ and $\hat{g}_a (k)$ are independent. Following \eqref{eq:apo_rct_est_categ}, we can then write
\begin{align}
    \E[\hat{\mu}_{a}^{\tn{OM}}] &= \E \left[ \sum_{k=1}^K \hat{p}_{s=0} (k) \hat{g}_a (k) \right] \nonumber \\
    &=\sum_{k=1}^K \E \left[\hat{p}_{s=0} (k) \hat{g}_a (k) \right] \nonumber \\
    &=\sum_{k=1}^K \E \left[ \hat{p}_{s=0} (k) \right] \E \left[ \hat{g}_a (k) \right] \nonumber \\
    &=\sum_{k=1}^K p_{s=0} (k) \E \left[ Y^a \mid X=k,S=0 \right] \label{eq:unb_sp} \\
    &=\E[Y^a \mid S = 0] \tag{law of total expectation} \\
    &=\mu_a, \label{eq:unb_cat} 
\end{align}
where \eqref{eq:unb_sp} follows from the unbiasedness of sample proportion and \eqref{eq:lem_varLb_1}. 

Next, by the law of total variance we write
\begin{align}
    \V\left(\hat{\mu}_{a}^{\tn{OM}}\right) &= \E_{{\cal D}_0} \left[\V_{{\cal D}_1} (\hat{\mu}_{a}^{\tn{OM}} \mid {\cal D}_0)\right] + \V_{{\cal D}_0} \left(\E_{{\cal D}_1} \left[\hat{\mu}_{a}^{\tn{OM}}\mid {\cal D}_0\right]\right). \label{eq:lotv} 
\end{align}
We start with the first term.
\begin{align}
    \E_{{\cal D}_0} \left[\V_{{\cal D}_1} (\hat{\mu}_{a}^{\tn{OM}} \mid {\cal D}_0)\right] &=\E_{{\cal D}_0} \left[ \V_{{\cal D}_1} \left( \sum_{k=1}^K \hat{p}_{s=0} (k) \hat{g}_a (k) \Big\lvert {\cal D}_0 \right) \right] \nonumber \\
    & =\E_{{\cal D}_0} \left[  \sum_{k=1}^K \hat{p}_{s=0}^2 (k) \V_{{\cal D}_1} \left( \hat{g}_a (k) \big\lvert {\cal D}_0 \right) \right] \label{eq:lem_varLb_2}\\
    & = \sum_{k=1}^K \E_{{\cal D}_0} \left[ \hat{p}_{s=0}^2 (k) \V_{{\cal D}_1} \left( \hat{g}_a (k) \big\lvert {\cal D}_0 \right) \right] \nonumber \\
    & = \sum_{k=1}^K \E_{{\cal D}_0} \left[ \hat{p}_{s=0}^2 (k) \right]  \V_{{\cal D}_1} \left( \hat{g}_a (k) \right) \label{eq:lem_varLb_3} \\
    & = \sum_{k=1}^K \left( \E _{{\cal D}_0}\left[ \hat{p}_{s=0} (k) \right]^2 + \underbrace{\V_{{\cal D}_0} \left( \hat{p}_{s=0} (k) \right)}_{\xrightarrow{n_0 \to \infty} 0 } \right) \V_{{\cal D}_1} \left( \hat{g}_a (k) \right) \label{eq:spav} \\
    & \approx \sum_{k=1}^K p^2_{s=0} (k)\frac{\sigma^2_{a,k}}{n_{s=1,a,k}}, \label{eq:var_cat}
\end{align}
where \eqref{eq:lem_varLb_2} holds since the participants in different groups are independent of each other and $\hat{p}_{s=0}^2 (k)$ is no longer random after conditioning on ${\cal D}_0$. Similar to above, \eqref{eq:lem_varLb_3} holds since $\hat{g}_a (k)$ is independent of the target sample ${\cal D}_0$ and therefore $\hat{p}_{s=0}^2 (k)$. \eqref{eq:spav} follows after writing the variance of the sample proportion 
\begin{equation} \label{eq:spav2}
    \V_{{\cal D}_0} \left( \hat{p}_{s=0} (k) \right) = \frac{p_{s=0} (k) \left(1 - p_{s=0} (k) \right)}{n_0}.
\end{equation}
For the second term we write
\begin{align}
    \V_{{\cal D}_0} \left(\E_{{\cal D}_1} \left[\hat{\mu}_{a}^{\tn{OM}}\mid {\cal D}_0\right]\right) &= \V_{{\cal D}_0} \left( \E_{{\cal D}_1} \left[ \sum_{k=1}^K \hat{p}_{s=0} (k) \hat{g}_a (k) \mid {\cal D}_0 \right] \right) \nonumber \\
    &= \V_{{\cal D}_0} \left( \sum_{k=1}^K \hat{p}_{s=0} (k) \E_{{\cal D}_1} \left[ \hat{g}_a (k) \mid {\cal D}_0 \right] \right) \nonumber \\
    &= \V_{{\cal D}_0} \left( \sum_{k=1}^K \hat{p}_{s=0} (k) \E_{{\cal D}_1} \left[ \hat{g}_a (k) \right] \right) \nonumber \\
    &= \V_{{\cal D}_0} \left( \sum_{k=1}^K \hat{p}_{s=0} (k)  \E \left[ Y^a \mid X=k,S=0 \right]  \right) \nonumber \\
    &= \sum_{k=1}^K \underbrace{\V_{{\cal D}_0} \left( \hat{p}_{s=0} (k) \right)}_{\xrightarrow{n_0 \to \infty} 0 }  \E \left[ Y^a \mid X=k,S=0 \right]^2  \label{eq:asvar2} \\
    &\approx 0, \label{eq:asvar0}
\end{align}
where \eqref{eq:asvar2} follows again from \eqref{eq:spav2}. Combining \eqref{eq:lotv}, \eqref{eq:var_cat}, and \eqref{eq:asvar0}, we have, as $n_0$ goes to infinity,
\begin{equation}\label{eq:asvar1}
    \V\left(\hat{\mu}_{a}^{\tn{OM}}\right) \approx \sum_{k=1}^K p^2_{s=0} (k)\frac{\sigma^2_{a,k}}{n_{s=1,a,k}}
\end{equation}
Finally, we have
\begin{align*}
    \E[(\hat{\mu}_{a}^{\tn{OM}} - \mu_a)^2] &= \underbrace{\E[(\hat{\mu}_{a}^{\tn{OM}} - \mu_a)]^2}_{0~\tn{by}~\eqref{eq:unb_cat}} + \V\left(\hat{\mu}_{a}^{\tn{OM}}\right) \\
    &\approx \sum_{k=1}^K p^2_{s=0} (k)\frac{\sigma^2_{a,k}}{n_{s=1,a,k}}, \tag{by \eqref{eq:asvar1}}
\end{align*}
and we are done.
\end{proof}
\mseLem*
\begin{proof}
    \begin{align} 
    \E[ \left( \hat{\mu}_{a}^{\tn{OM}} - \mu_a \right)^2 ] = \underbrace{\left(\E\left[ \hat{\mu}_{a}^{\tn{OM}} \right]  - \mu_a \right)^2}_{\tn{Bias}^2} + \underbrace{\V (\hat{\mu}_{a}^{\tn{OM}})}_{\tn{Variance}}. \label{eq:lemmse}
    \end{align}
    We will start with the bias term. Note that, under Assumptions \ref{asm:cons}-\ref{asm:pos}, we have
    \begin{align}
        \mu_a 
        &= \E[ \E[ Y \mid X, S=1, A=a] \mid S=0] \tag{see \eqref{eq:apo_iden}} \nonumber \\
        &= \E[ g_a (X)\mid S=0] \tag{by definition, see \eqref{eq:mof}} \nonumber \\
        &= \E_{X \sim P_0} [ g_a (X)]. \label{eq:train1}
    \end{align}
    with a manipulation of notation at the final step. Once $\hat{\theta}$ is estimated from the trial sample ${\cal D}_1$ via an algorithm ${\cal A}$, $\hat{\mu}_{a}^{\tn{OM}}$ is calculated by effectively taking a sample mean of $g_a(X_i; \hat{\theta})$ for the covariates $X_i$ in the target sample ${\cal D}_0$. We can write,
    %
    %
   \begin{align}
        \E[ \hat{\mu}_{a}^{\tn{OM}} ] 
        &= \E \left[ \frac{1}{n_0} \sum_{i=1}^n \ind{S_i = 0} g_a (X_i; \hat{\theta}) \right] \nonumber \\
        &= \E \left[ \frac{1}{n_0} \sum_{X_i \in {\cal D}_0} g_a (X_i; \hat{\theta}) \right] \nonumber \\
        &= \E_{\hat{\theta} \sim P_1} \left[ \E_{{\cal D}_0} \left[ \frac{1}{n_0} \sum_{X_i \in {\cal D}_0} g_a (X_i; \hat{\theta}) \Big\lvert \hat{\theta} \right] \right] \nonumber \\
        &= \E_{\hat{\theta} \sim P_1} \left[ \E_{{\cal D}_0} \left[ \frac{1}{n_0} \sum_{X_i \in {\cal D}_0} g_a (X_i; \hat{\theta})  \right] \right] \nonumber \\
        &= \E_{\hat{\theta} \sim P_1} \left[ \E_{X \sim P_0} \left[ g_a (X; \hat{\theta})  \right] \right] \nonumber \\
        &= \E_{X \sim P_0} \left[ \E_{\hat{\theta} \sim P_1} \left[ g_a (X; \hat{\theta})  \right] \right]. \label{eq:train2}
    \end{align}
    since $X_i \in {\cal D}_0$ are i.i.d and independent of $\hat{\theta}$. Combining \eqref{eq:train1} and \eqref{eq:train2} we write
    \begin{align}
        \tn{Bias} 
        &=\E\left[ \hat{\mu}_{a}^{\tn{OM}} \right]  - \mu_a \nonumber \\
        &=\E_{X \sim P_0} \Big[ \E_{\hat{\theta} \sim {\cal A} (P_1)}[ g_a (X; \hat{\theta})] - g_a (X) \Big]. \label{eq:biasderived}  
    \end{align}
    We continue with the variance term by invoking the law of total variance.
    \begin{align}
        \tn{Variance} 
        &=\V \left(\hat{\mu}_{a}^{\tn{OM}} \right) \nonumber \\
        &=\V_{\hat{\theta} \sim {\cal A} (P_1), {\cal D}_0}\left( \frac{1}{n_0} \sum_{X_i \in {\cal D}_0} g_a (X_i; \hat{\theta}) \right) \nonumber \\
        &=\V_{\hat{\theta} \sim {\cal A} (P_1)} \left( \E_{{\cal D}_0} \left[ \frac{1}{n_0} \sum_{X_i \in {\cal D}_0} g_a (X_i; \hat{\theta})\Bigg\lvert \hat{\theta} \right]  \right) \nonumber \\
        &\hspace{50pt}+\underbrace{\E_{\hat{\theta} \sim {\cal A} (P_1)} \left[ \underbrace{\V_{{\cal D}_0}  \left( \frac{1}{n_0} \sum_{X_i \in {\cal D}_0}  g_a (X_i; \hat{\theta}) \Bigg\lvert \hat{\theta} \right)}_{\xrightarrow{n_0 \to \infty} 0 } \right]}_{\approx 0} \nonumber \\
        &\approx \V_{\hat{\theta} \sim {\cal A} (P_1)}  \left(  \E_{{\cal D}_0} \left[ \frac{1}{n_0} \sum_{X_i \in {\cal D}_0} g_a (X_i; \hat{\theta}) \right]  \right) \nonumber \\
        &=\V_{\hat{\theta} \sim {\cal A} (P_1)}\left( \E_{X \sim P_0} \left[ g_a (X; \hat{\theta}) \right] \right), \label{eq:varderived}
    \end{align}
    where in the third equality, the variance of the sample mean vanishes for large $n_0$. Last two steps follow since the target sample ${\cal D}_0$ and $\hat{\theta}$ are independent and $X_i \in {\cal D}_0$ are i.i.d. Plugging \eqref{eq:biasderived} and \eqref{eq:varderived} into the definition of the MSE in \eqref{eq:lemmse}, we are done.
\end{proof}
\altiden*
\begin{proof}
    Recall that we have
    \begin{equation} \label{eq:alt_iden_sup}
        \mu_a = \E \left[ f_a(X) \mid S=0 \right] - \E \left[f_a(X) - Y^a \mid S=0 \right].
    \end{equation}
    Note that the prediction model $f_a$ is fixed. Therefore we have
    \begin{equation}
        \E [f_a(X) \mid S=0] = \E_{X \sim P_0} [f_a(X)], \label{eq:lai1}
    \end{equation}
    which is only a change of notation. Further, conditioned on $X$, $f_a(X)$ is no more random, We can then write
    \begin{align}
        \E [f_a(X) - Y^a \mid S=0]
        &= \E_{X \sim P_0} \big[ \E[ f_a(X) - Y^a \mid X, S=0] \big] \nonumber \\
        &= \E_{X \sim P_0} \big[ \E[ f_a(X) - Y^a \mid X, S=1] \big] \label{eq:lai2} \\ 
        &= \E_{X \sim P_0} \big[ \E[ f_a(X) - Y^a \mid X, S=1, A=a] \big] \label{eq:lai3} \\ 
        &= \E_{X \sim P_0} \big[ \E[ f_a(X) - Y \mid X, S=1, A=a] \big] \label{eq:lai4} \\ 
        &= \E_{X \sim P_0} \big[ \E[ Z \mid X, S=1, A=a] \big], \label{eq:lai5}     
    \end{align}
    where \eqref{eq:lai2} is due to Assumptions~\ref{asm:nuem} and \ref{asm:pos}, \eqref{eq:lai3} is due to Assumptions \ref{asm:nuc} and \ref{asm:pot}, and \eqref{eq:lai4} is due to Assumption~\ref{asm:cons}.

    Combining \eqref{eq:alt_iden_sup}, \eqref{eq:lai1}, and \eqref{eq:lai5} completes the proof.
\end{proof}
\ppommse*
\begin{proof}
    We have, by \Cref{lemma:alt_iden},
    \begin{align}
        \mu_a &= \E[f_a(X) - b_a(X) \mid S=0] \nonumber \\
        &= \E[f_a(X) - b_a(X) \mid S=0]. \label{eq:ppom0}
    \end{align}
    where
    \begin{equation*}
        b_a(X) = \E[ f_a(X) - Y \mid X, S=1, A=a].
    \end{equation*}
    We consider a parametric estimator $b_a(X; \hat{\gamma})$ where $\gamma$ is estimated from the instance-wise prediction errors $f_a(X) - Y$ in the trial sample ${\cal D}_1$.
    
    We will follow the same steps in the proof of \Cref{theorem:mseLem} for the most part.
    \begin{align} 
        \E[ \left( \hat{\mu}_{a}^{\tn{ABC}} - \mu_a \right)^2 ] = \underbrace{\left(\E\left[ \hat{\mu}_{a}^{\tn{ABC}} \right]  - \mu_a \right)^2}_{\tn{Bias}^2} + \underbrace{\V (\hat{\mu}_{a}^{\tn{ABC}})}_{\tn{Variance}}. \label{eq:ppom1}
    \end{align}
    We have
    \begin{align}
        \E[ \hat{\mu}_{a}^{\tn{ABC}} ] 
        &= \E \left[ \frac{1}{n_0} \sum_{i=1}^n \ind{S_i = 0}  \big(f_a (X_i) - b^a(X_i; \hat{\gamma}) \big) \right] \nonumber \\
        &= \E \left[ \frac{1}{n_0} \sum_{X_i \in {\cal D}_0} f_a (X_i) - b^a(X_i; \hat{\gamma}) \right]. \label{eq:ppom2}
    \end{align}
    Since the sample mean is unbiased, it follows that
    \begin{equation}
        \E \left[ \frac{1}{n_0} \sum_{X_i \in {\cal D}_0} f_a (X_i) \right] = \E_{X \sim P_0} \left[ f_a(X) \right]. \label{eq:ppom3}
    \end{equation}
    Next, via the same machinery that derives \eqref{eq:train2}, we have
    \begin{align}
        \E \left[ \frac{1}{n_0} \sum_{X_i \in {\cal D}_0} b^{a} (X_i, \hat{\gamma}) \right] 
        = \E_{X \sim P_0} \Big[ \E_{\hat{\gamma} \sim P_1} \left[ b_a (X; \hat{\gamma}) \right] \Big]. \label{eq:ppom4}
    \end{align}
    By \eqref{eq:ppom2}, \eqref{eq:ppom3}, and \eqref{eq:ppom4}, we observe
    \begin{align}
        \E[ \hat{\mu}_{a}^{\tn{ABC}} ] 
        = \E_{X \sim P_0} \Big[ f_a(X) - \E_{\hat{\gamma} \sim P_1} \left[ b_a (X; \hat{\gamma}) \right] \Big], \label{eq:ppom5}
    \end{align}
    which leads to, in combination with \eqref{eq:ppom0}
    \begin{align}
        \tn{Bias} &= \E\left[ \hat{\mu}_{a}^{\tn{ABC}} \right]  - \mu_a \nonumber \\
        &=\E_{X \sim P_0} \Big[ b_a(X) - \E_{\hat{\gamma} \sim P_1} \left[ b_a (X; \hat{\gamma}) \right] \Big]. \label{eq:ppom6}
    \end{align}
    We continue with the variance term.
    \begin{align}
        \tn{Variance} &=\V \left(\hat{\mu}_{a}^{\tn{ABC}} \right) \nonumber \\
        &=\V \left( \frac{1}{n_0} \sum_{X_i \in {\cal D}_0} f_a (X_i) - b^a(X_i; \hat{\gamma}) \right) \nonumber \\
        &=\underbrace{\V_{{\cal D}_0} \left( \frac{1}{n_0} \sum_{X_i \in {\cal D}_0} f_a (X_i) \right)}_{\xrightarrow{n_0 \to \infty} 0 } \nonumber \\
        &\hspace{20pt}+ \V_{\hat{\gamma} \sim {\cal A} (P_1), {\cal D}_0} \left( \frac{1}{n_0} \sum_{X_i \in {\cal D}_0} b^{a} (X_i; \hat{\gamma}) \right)  \nonumber \\
        &\approx \V_{\hat{\gamma} \sim {\cal A} (P_1)}\left( \E_{X \sim P_0} \left[ b_a (X; \hat{\gamma}) \right] \right). \label{eq:ppom7}
    \end{align}
    The decomposition in third equality is due to the independence of the models predictions $f_a(X)$ and errors $f_a(X) - Y$ ($\hat{\gamma}$ is derived using the latter only). Finally, \eqref{eq:ppom7} follows through the same machinery that derives \eqref{eq:varderived}.
    
    We are done after plugging \eqref{eq:ppom6} and \eqref{eq:ppom7} into \eqref{eq:ppom1}.
\end{proof}
\subsection{Polynomial Ridge Regression} \label{app:prr}
We consider polynomial ridge regression in the trial sample ${\cal D}_1$  using Legendre polynomials up to degree $d'$, to fit the outcome model and bias model estimates, $\hat{g}_a$ and $\hat{b}_a$ which results in the following fits with an appropriately chosen penalty parameter $\lambda$ \cite{wainwright2019high}.
\begin{align}
    \hat{g}_a \in \argmin_{g \in {\cal F }(d')} \left \{ \frac{1}{n_1} \sum_{i=1}^{n_1} (Y_i - g(X_i))^2 \right \}, \label{eq:oracle_est_ga} \\
    \hat{b}_a \in \argmin_{b \in {\cal F}(d')} \left \{ \frac{1}{n_1} \sum_{i=1}^{n_1} (Z_i - b(X_i))^2 \right \}. \label{eq:oracle_est_ba}
\end{align}
where
\begin{equation} \label{eq:class_of_poly}
    {\cal F} (d') \coloneqq \left\{ \sum_{k=0} \beta_k \phi_k (X) \mid \sum_{k=0}^{d'} \beta^2_k \leq 1 \right\},
\end{equation}
is the class of polynomials up to degree $d'$ with bounded norm. The results then follow from the oracle inequalities derived for the orthogonal basis approximation problem in Example 13.14, Section 13.3 of \cite{wainwright2019high}.
\subsection{MSE Approximation for the Augmented Outcome Modeling Approach} \label{app:aomproof}
\begin{restatable}[]{theorem}{aommse}
    \label{theorem:aommse}
        Suppose that Assumptions~\ref{asm:cons}-\ref{asm:pos} hold. For large $n_0$, the MSE of $\hat{\mu}_{a}^{\tn{AOM}}$ in \eqref{eq:apo_est_com} can be approximated as
    \begin{align} 
        \E[(\hat{\mu}_{a}^{\tn{AOM}} - \mu_a)^2] \approx \E_{X \sim P_0} \big[ \E_{\hat{\beta} \sim {\cal A} (P_1)}[ h_a (\tilde{X}; \hat{\beta})] - h_a (\tilde{X}) \big]^2 +\V_{\hat{\beta} \sim {\cal A} (P_1)}\big( \E_{X \sim P_0} [ h_a (\tilde{X}; \hat{\beta}) ] \big). \label{eq:aom_mse_p2}
    \end{align}
\end{restatable}

\begin{proof}
    The proof follows from the same steps in the proof of \Cref{theorem:mseLem}.
\end{proof}
\subsection{Doubly-Robust Estimation} \label{app:drest}
In order to leverage the prognostic model $f_a(X)$ in the analysis, we can proceed with two identifications of $\mu_a$, \eqref{eq:altidenlemma} and \eqref{eq:apo_iden_3}, for which we considered estimators based {\em only} on regression functions (\eqref{eq:ppom} and \eqref{eq:apo_est_com}). However, in practice, we can directly use the so-called doubly-robust (DR) estimators for \eqref{eq:altidenlemma} and \eqref{eq:apo_iden_3}, which have several desirable properties.

In addition to a regression function component, DR estimators also have {\em weighting} function components, which, in our case, are the probability of trial enrollment, $P(S=1 \mid X)$, and the probability of treatment assignment in the trial, $P(A=a \mid X, S=1)$. Estimators based only on {\em weighting} models are also available but will not be covered here in the interest of space. \cite{dahabreh2020extending} derives a generic DR estimator for the functional
\begin{equation} \label{eq:func}
    \E_{X \sim P_0}[ \E[ Y \mid X, S=1, A=a]],
\end{equation}
which we can directly adopt and use to estimate \eqref{eq:apo_iden_3} and the second term of \eqref{eq:altidenlemma}. Estimating the first term of \eqref{eq:altidenlemma} remains unchanged as the average of predictions $f_a(X)$ in the target sample. We make the following definitions.
\begin{align}
    p &= P(S=1). \label{eq:marg_tp} \\
    p(X) &= P(S=1 \mid X). \label{eq:cond_tp} \\
    \pi_a(X) &= P(A=a \mid X, S=1), \label{eq:prop_tr}
\end{align}
and denote by $\hat{p}$, $\hat{p} (X)$, and $\hat{\pi}_a (X)$ their estimates. Note that in order to use DR estimators, one now needs to fit those functions using the composite sample ${\cal D}$. Next we give the DR estimator for \eqref{eq:altidenlemma} and \eqref{eq:apo_iden_3}.
\begin{align} 
    \hat{\mu}_{a}^{\tn{DR-ABC}} &= \frac{1}{n_0} \sum_{i=1}^n \ind{S_i = 0} f_a (X_i) \nonumber \\
    &\hspace{10pt}+ \frac{1}{n (1-\hat{p})} \sum_{i=1}^n \Big(  \ind{S_i = 0}\hat{b}_a (X_i) + \ind{S_i = 1, A=a} \frac{1 - \hat{p} (X_i)}{\hat{p} (X_i) \hat{\pi}^a (X_i)} \big(Z_i - \hat{b}_a (X_i)\big) \Big). \label{eq:drabc}
\end{align}
\begin{align} 
    \hat{\mu}^{a}_{\tn{DR-PA}} &= \frac{1}{n (1-\hat{p})} \sum_{i=1}^n \Big(  \ind{S_i = 0}\hat{h}_a (\tilde{X}_i) + \ind{S_i = 1, A=a} \frac{1 - \hat{p} (X_i)}{\hat{p} (\tilde{X}_i) \hat{\pi}^a (\tilde{X}_i)} \big(Y - \hat{h}_a (\tilde{X}_i)\big) \Big). \label{eq:drpa}
\end{align}
An essential property of DR estimators is that consistent estimation of {\em either} the regression or weighting functions guarantees consistent estimation of $\mu_a$, hence the name \enquote{doubly-robust}. Beyond this DR property, however, they have other desirable properties (under certain regularity conditions or cross-fitting techniques \cite{chernozhukov2018double}) such as asymptotical efficiency and normality, which enable one to construct confidence intervals beyond point prediction and allow for, {\em e.g.}, calculating p-values and testing hypotheses. We refer the interested reader to \cite{kennedy2023towards} for a unifying overview of the theory around the DR estimators, their properties, and how to construct them for different estimands of interest. We present empirical results for the DR estimators in \Cref{app:ipw_dr}.

\section{Additional Experimental Results} \label{app:aer}
\subsection{Bias-Variance Tradeoff} \label{app:bvt}
We do not plot the bias and variance terms for $\hat{\mu}^{\tn{OS-OM}}_1$. It has minimal ($\approx 0$) variance as one applies the observational predictor directly to the target sample, and nothing is fit from the small trial data. Since the target sample is taken to be large, the variance in $\hat{\mu}^{\tn{OS-OM}}_1$ is negligible. Almost all of its MSE (see \Cref{fig:synthetic-rmse}) consists of the bias, which results from hidden confounding introduced by concealing the $U$ variable (see \Cref{fig:dgp}).

In \Cref{fig:synthetic-biassq}, we see that the bias resulting from estimating the outcome function $g_1(X)$ from the trial sample is very large with a small model. Although it decreases with a larger model as expected, we see in \Cref{fig:synthetic-var} that the variance quickly explodes when the trial size is small ($n_1 = 200$) and $g_1 (X)$ is complex and has high intrinsic variation. When$n_1 = 1000$, the variance terms significantly decrease by a factor of $10$, and the RMSE of $\hat{\mu}_1^{\tn{OM}}$ becomes comparable to prediction-powered estimators when higher-degree polynomials are fit for $\hat{g}_1 (X)$.

Our approaches leveraging the additional predictor have smaller bias and variance terms. The difference is the most significant when the trial is small and the outcome function is complex.
\begin{figure*}[ht]
    \centering
    \includegraphics[width=\linewidth]{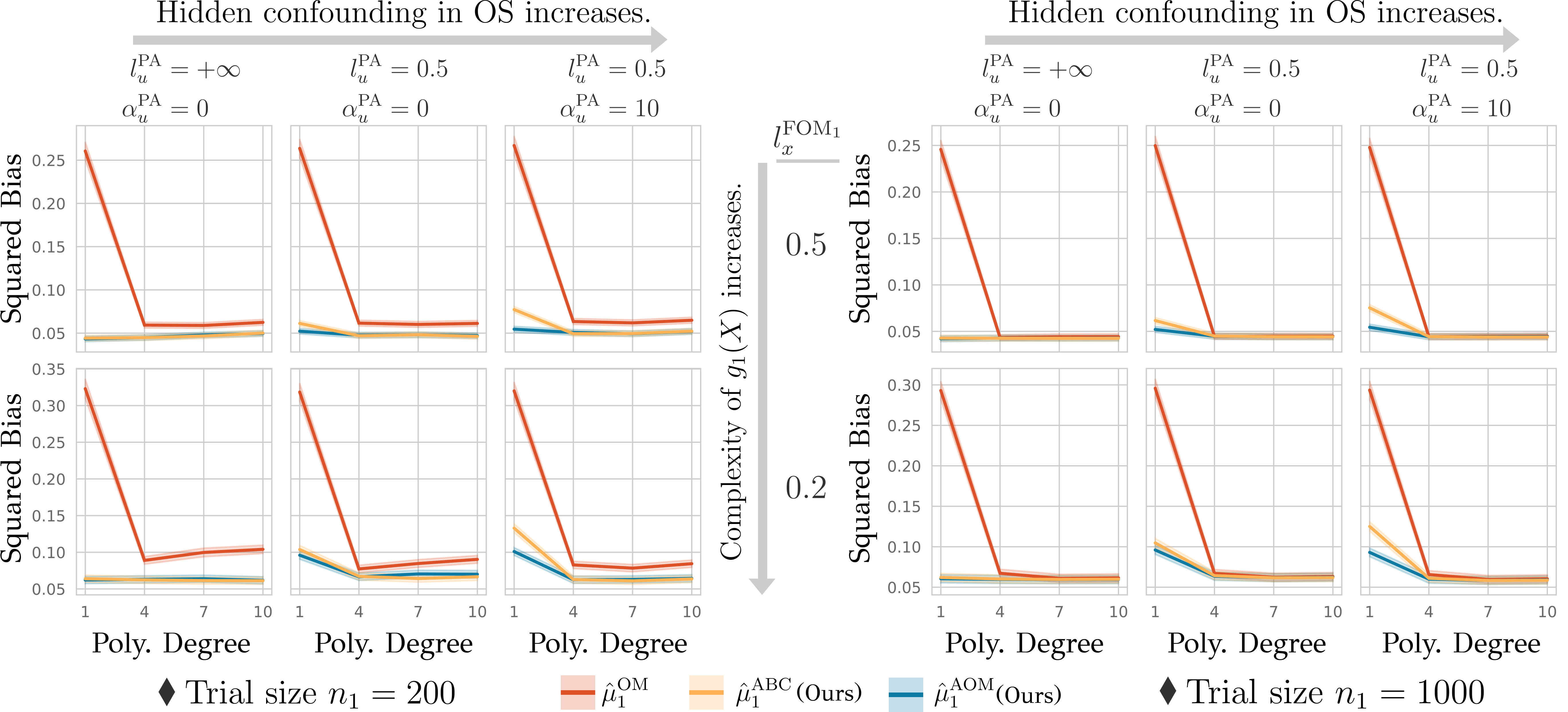}
    \caption{Convention same as \Cref{fig:synthetic-rmse}. Average squared bias of each estimator is plotted.}
    \label{fig:synthetic-biassq}
\end{figure*}

\begin{figure*}[ht]
    \centering
    \includegraphics[width=\linewidth]{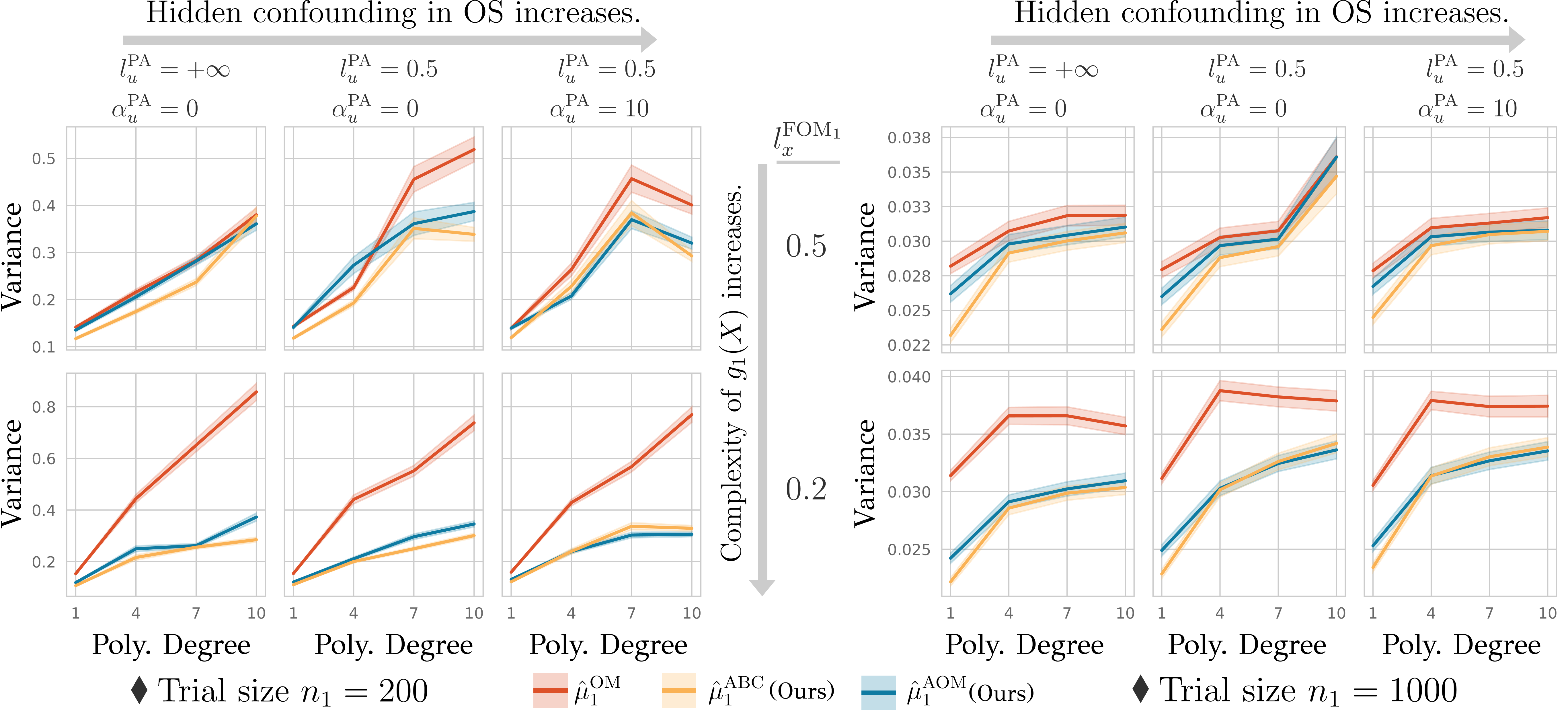}
    \vspace{-10pt}
    \caption{Convention same as \Cref{fig:synthetic-rmse}. Average variance of each estimator is plotted.}
    \label{fig:synthetic-var}
\end{figure*}
\subsection{IPW and DR-based Estimators} \label{app:ipw_dr}
In \Cref{fig:synthetic-rmse-ipw-dr}, we include the generalization RMSE for the doubly-robust (DR) and inverse propensity weighting (IPW) estimators. DR versions of our methods are given in \Cref{app:drest}. Baseline IPW and DR estimators are detailed in \cite{dahabreh2020extending}. The IPW estimator performs the worst due to the high variance in the propensity weight estimates, and the DR estimators perform similarly to the outcome-model (OM) estimators. \citet{dahabreh2020extending} report similar results. Finally, we note that as the sample size in the trial $n_{1}$ increases, the MSE of different estimators converge as before.
\begin{figure*}[b!]
    \centering
    \includegraphics[width=\linewidth]{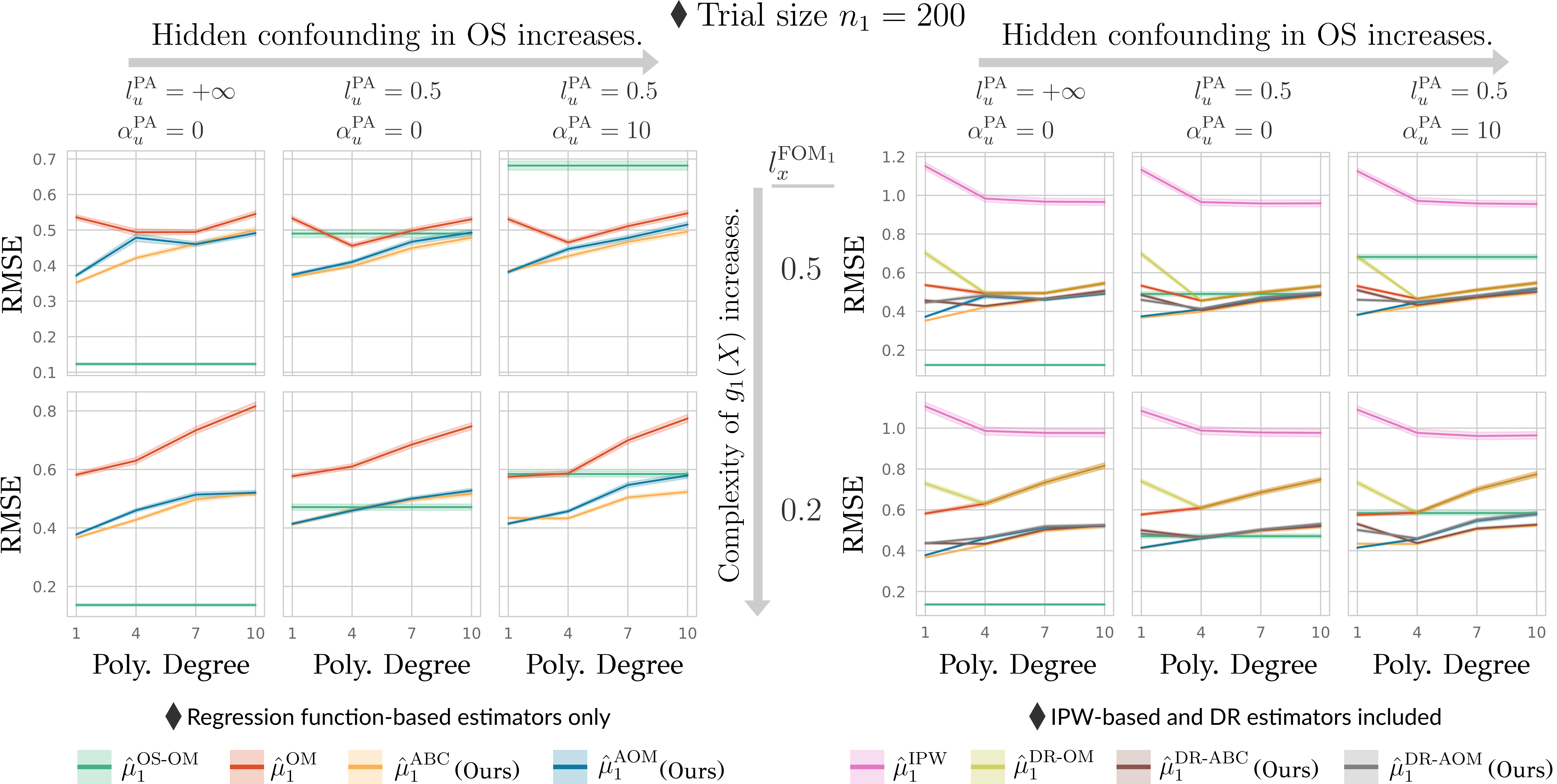}
    \vspace{-10pt}
    \caption{Generalization RMSE with DR and IPW estimators included.}
    \label{fig:synthetic-rmse-ipw-dr}
\end{figure*}

\subsection{Example Cases} \label{app:ec}
In Figures~\ref{fig:exc-4}-\ref{fig:exc-3}, we demonstrate several example cases where we plot the ground truth functions for the synthetic data-generating processes, an example trial sample that is used to fit $\hat{b}_1 (X)$, $\hat{g}_1 (X)$, $\hat{h}_1 (X)$ (plotted the linear fits only for simplicity, {\em i.e.}, 1st-degree polynomials), and the observational predictor $f_1(X)$. 

We aim to demonstrate how the outcome function $g_1(X)$ becomes \enquote{wiggly} as $l_x^{\tn{FOM}_1}$ decreases and has more rapid turns representing higher-order polynomials. Further, we see that as the hidden confounding increases, {\em i.e.}, as we move along the $x$-axis of plots, the bias of the observational predictor, $b_1(X) = g_1(X) - f_1(X)$ also increases and becomes a \enquote{higher-norm} function, decreasing the utility of leveraging observational data and increasing the RMSE. 

As we referred to earlier, one can see in Figures~\ref{fig:exc-2} and \ref{fig:exc-3}, for the cases with $l_x^{\tn{FOM}_1} = 0.2$ (complex outcome function $g_1(X)$) and $(l_u^{\tn{PA}} = 0.5, \alpha_u^{\tn{FOM}_1} = 10)$ (large hidden confounding), the bias function $b_1(X)$ is also a complex function with high norm, and the RMSE of the prediction-powered approaches are not significantly better than the baseline estimator $\hat{\mu}_{1}^{\tn{OM}}$.
\subsection{Using (Generalized) Linear Models in the Data-generating Process} \label{app:glmexp}
We sincerely thank Reviewer xwX2 for taking the time during the author-reviewer discussion phase to provide the initial codebase for the results presented here. Instead of generating the outcome and propensity score functions using GPs, we use polynomial models.

The full outcome model under treatment $A=1$ is specified as a 5-th order polynomial with parameter $\beta$. Precisely, we set
\begin{equation}
   \tn{FOM}_1 (X,U) = \beta_0 + \beta^X_1 X + \dots + \beta^X_5 X^5 + \gamma \left(\beta^U_1 U + \dots + \beta^U_5 U^5 + \beta^{XU}_1 XU + \dots + \beta^{XU}_5 (XU)^5 \right), \label{eq:linfom}
\end{equation}
and observe $Y^1 = \tn{FOM}_1 (X,U) + \epsilon$ where $\epsilon \sim {\cal N} (0, \sigma^2)$ for some $\sigma \in \mathbb{R}_+$ which we use to model the intrinsic variation in outcome observations. Larger values for $\sigma$ increase the risk of overfitting when the trial size $n_1$ is small. $\beta$ parameters characterize the complexity of the outcome function.

The probability of selection into the trial and the probability of treatment assignment in the OS are modeled as follows.
\begin{align}
    &P(S=1 \vert X, U) = \frac{1}{1 +  \exp \left(\lambda_0 + \lambda_1 X + \dots + \lambda_5 X^5\right)}.  \\
    &P(A=1 \vert S=2,X, U) \nonumber\\
    &\hspace{10pt}= \frac{1}{1 +  \exp \left(\alpha_0 + \alpha^X_1 X + \dots + \alpha^X_5 X^5 + \gamma \left(\alpha^U_1 U + \dots + \alpha^U_5 U^5 + \alpha^{XU}_1 XU + \dots + \alpha^{XU}_5 (XU)^5 \right) \right)}.
\end{align}
$\gamma \in \mathbb{R}_+$ determines the amount of hidden confounding in the OS, as $U$ is concealed. Further, $\lambda$ parameters characterize how weak the overlap is between the trial and target samples.

Briefly, larger values for $\beta$, $\gamma$, $\lambda$, and $\sigma$ parameters make the generalization task more challenging. In \Cref{tab:results}, we present the generalization MSEs under various settings. We always use $\alpha \sim {\cal N} (0, 1)$. We sample 100 ground-truth $\alpha, \beta$ parameters for each setting, make 100 independent runs for each ground-truth, and then present the average MSE values. We use both 1st and 5th order polynomials to fit the bias and outcome functions, $\hat{b}_1 (X)$ and $\hat{g}_1(X)$.
\begin{table}[h]
\centering
\begin{tabular}{>{\raggedright\arraybackslash}m{8cm}>{\centering\arraybackslash}m{1cm}>{\centering\arraybackslash}m{1cm}>{\centering\arraybackslash}m{1cm}>{\centering\arraybackslash}m{1cm}}
\toprule
\multirow{2}{*}{\makecell{Setting}} & \multicolumn{2}{c}{1st order poly. fit} & \multicolumn{2}{c}{5th order poly. fit} \\
\cmidrule(lr){2-3} \cmidrule(lr){4-5}
 & $\hat{\mu}_{1}^{\tn{ABC}}$ &$ \hat{\mu}_{1}^{\tn{OM}}$ & $\hat{\mu}_{1}^{\tn{ABC}}$ & $\hat{\mu}_{1}^{\tn{OM}}$ \\
\midrule
 $\gamma = 0, \quad \epsilon \sim {\cal N} (0,0.1^2), \quad \beta \sim {\cal N} (0,1), \quad \hspace{3pt} \lambda=1$ & \textbf{.0001} &.0092  & .0002 & .0002  \\
 $\gamma = 1, \quad \epsilon \sim {\cal N} (0,0.1^2), \quad \beta \sim {\cal N} (0,1), \quad \hspace{3pt} \lambda=1$ & \textbf{.0087} &.0183  & .0148 & .0148  \\
 $\gamma = 0, \quad \epsilon \sim {\cal N} (0,2^2), \qquad \hspace{-1pt} \beta \sim {\cal N} (0,1), \quad \hspace{3pt} \lambda=1$  & \textbf{.0044} &.0054  & .0066 & .0066  \\
 $\gamma = 0, \quad \epsilon \sim {\cal N} (0,2^2), \qquad \hspace{-1pt} \beta \sim {\cal N} (0, 2^2), \quad \hspace{-1pt}\lambda=1$  & \textbf{.0044} &.0082  & .0066 & .0066  \\
 $\gamma = 0, \quad \epsilon \sim {\cal N} (0,2^2), \qquad \beta \sim {\cal N} (0, 2^2), \quad \hspace{-1pt} \lambda=2$ & \textbf{.0048} &.0127  & .0151 & .0151  \\
 $\gamma = 1, \quad \epsilon \sim {\cal N} (0,2^2), \qquad \beta \sim {\cal N} (0, 2^2), \quad \hspace{-1pt} \lambda=2$ & \textbf{.0060} &.0141  & .0139 & .0139  \\
\bottomrule
\end{tabular}
\caption{Generalization MSEs using (generalized) linear models in the data-generating process.}
\label{tab:results}
\end{table}

\begin{figure*}[ht]
    \centering
    \includegraphics[width=\linewidth]{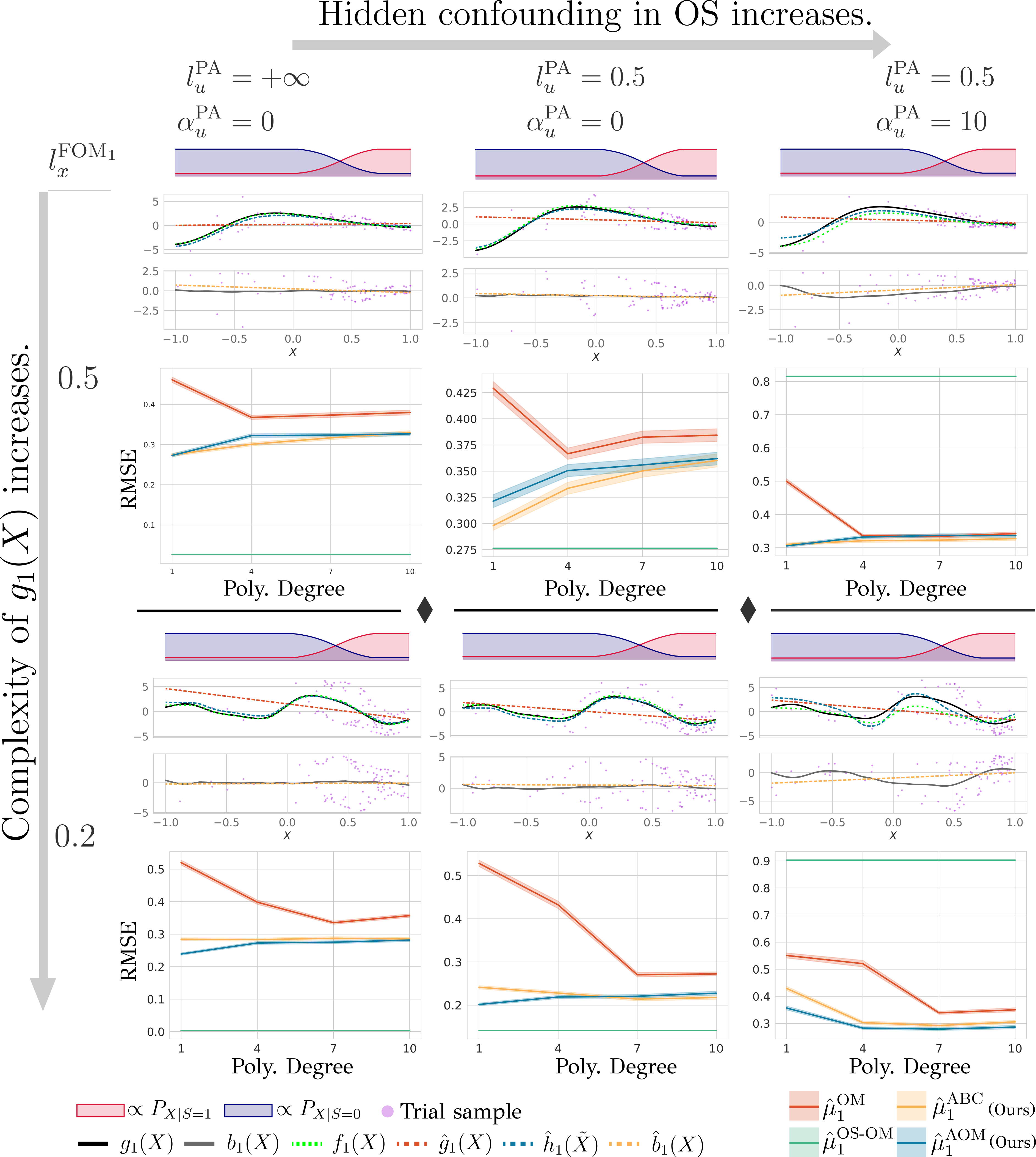}
    \caption{Example case 1.}
    \label{fig:exc-4}
\end{figure*}
\begin{figure*}[ht]
    \centering
    \includegraphics[width=\linewidth]{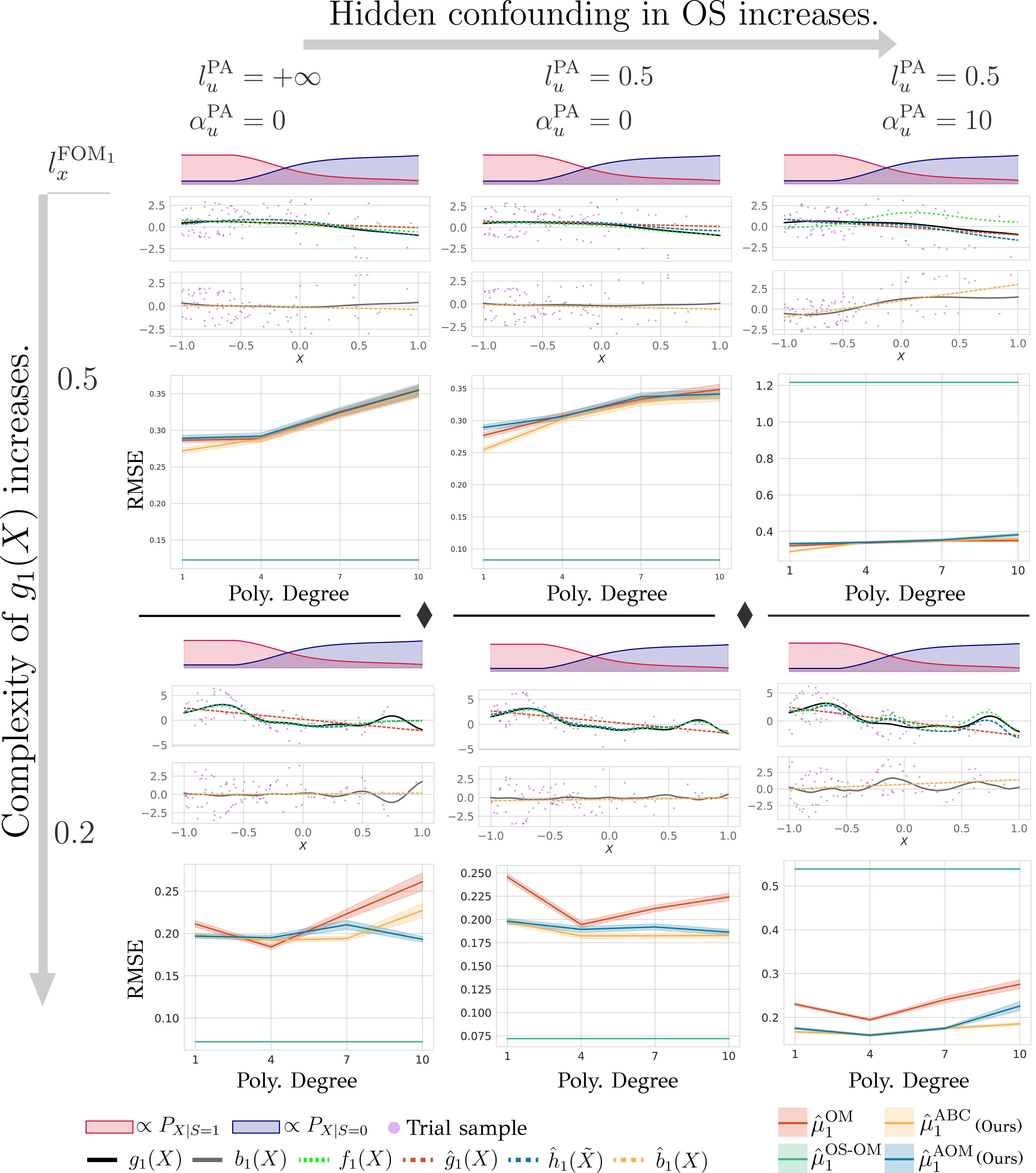}
    \caption{Example case 2.}
    \label{fig:exc-1}
\end{figure*}
\begin{figure*}[ht]
    \centering
    \includegraphics[width=\linewidth]{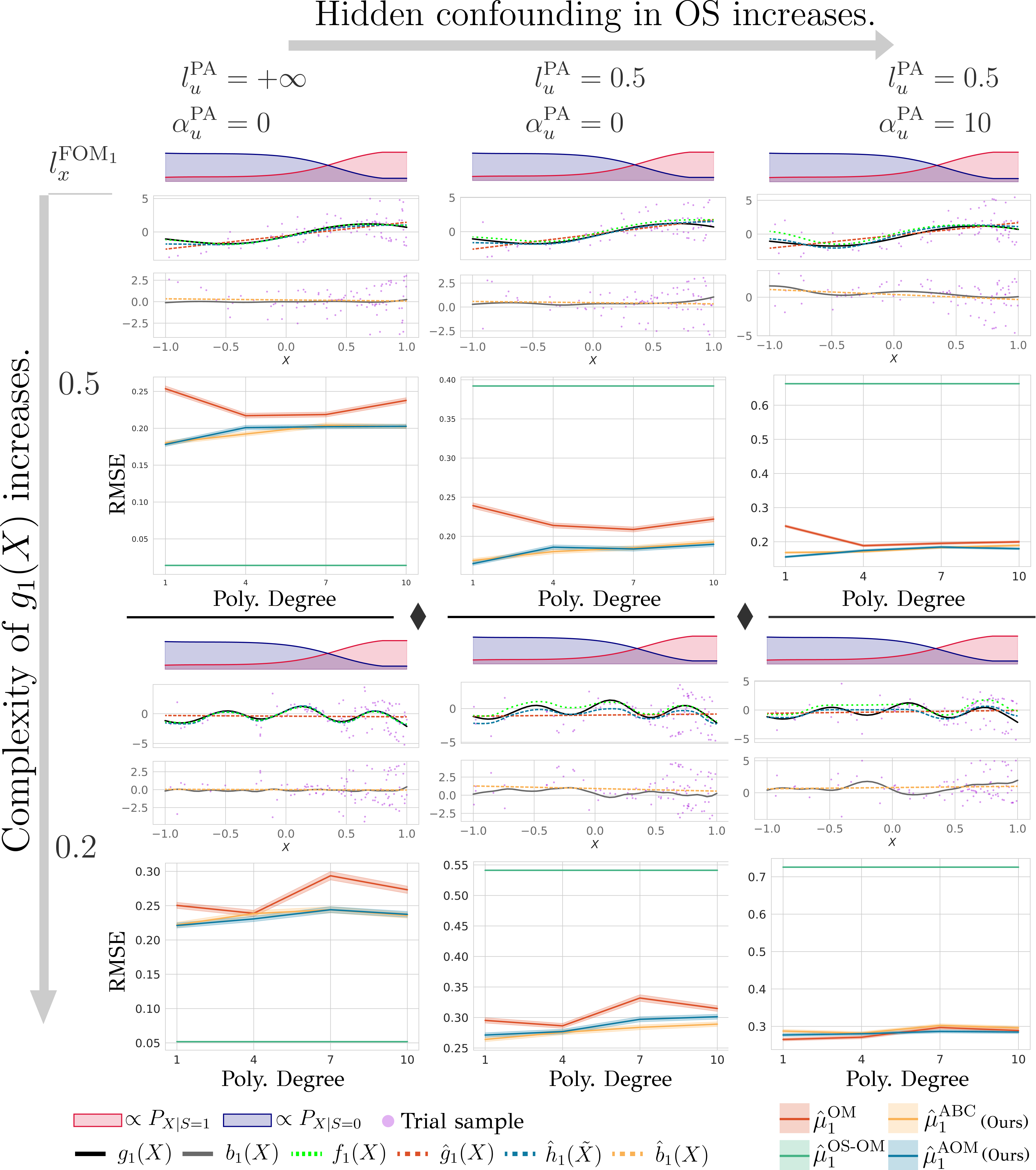}
    \caption{Example case 3.}
    \label{fig:exc-2}
\end{figure*}
\begin{figure*}[ht]
    \centering
    \includegraphics[width=\linewidth]{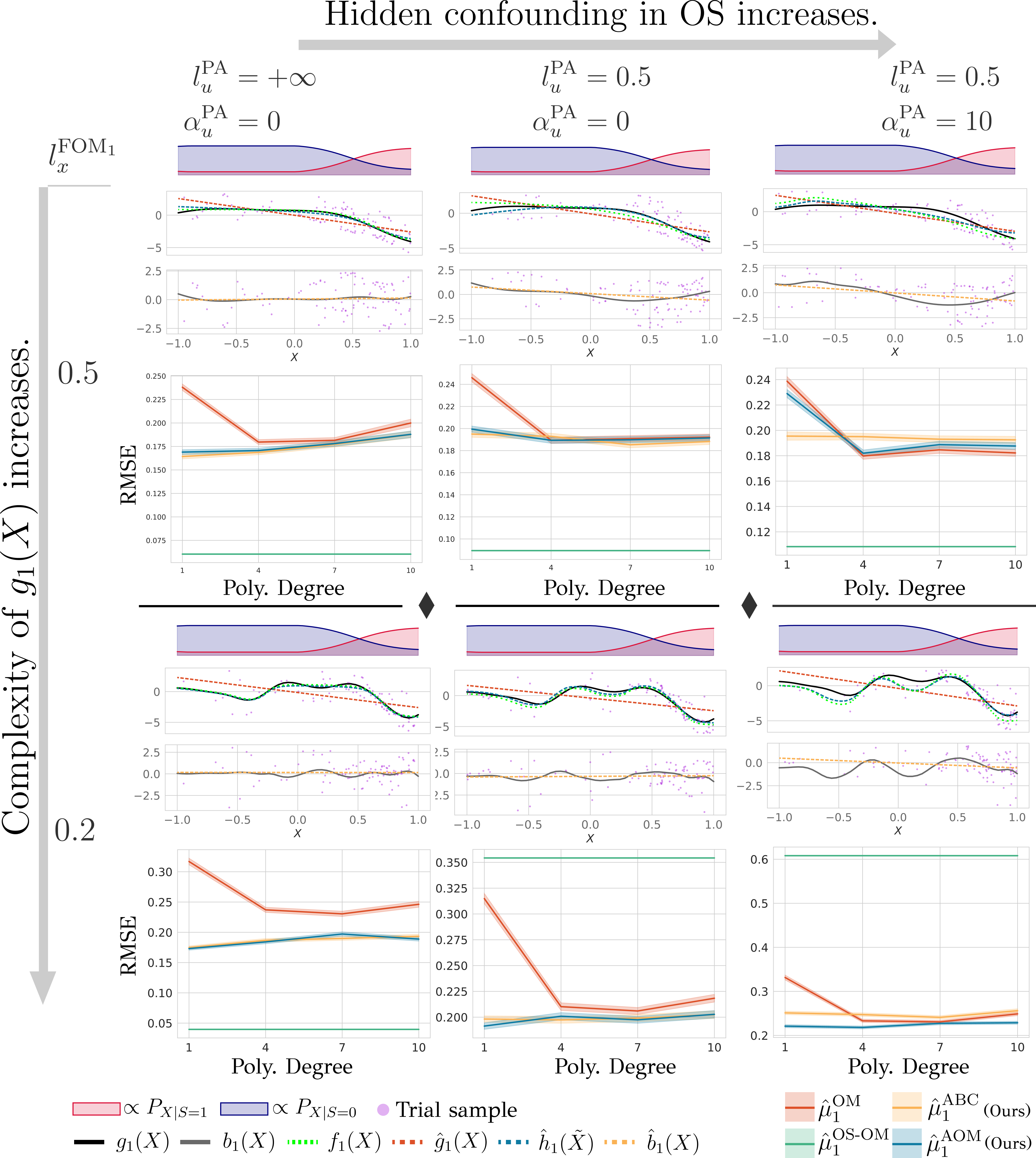}
    \caption{Example case 4.}
    \label{fig:exc-3}
\end{figure*}

\end{document}